\documentclass[10pt, journal, twocolumn]{IEEEtran}

\usepackage[utf8]{inputenc}
\usepackage{amsmath}
\usepackage{amsthm}
\usepackage{amsfonts}
\usepackage{amssymb}
\usepackage{graphicx}
\usepackage{graphics}
\usepackage{float}
\usepackage{tikz}
\usetikzlibrary{shapes, snakes, patterns}
\usepackage[export]{adjustbox}
\usepackage{subcaption}
\usepackage{algorithm}
\usepackage{algpseudocode}
\usepackage{epstopdf}
\newtheorem{theorem}{Theorem}[]
\newtheorem{lemma}[]{Lemma}
\newtheorem{corollary}[]{Corollary}
\newtheorem{remark}[]{Remark}
\newtheorem{Assumption}[]{Assumption}
\newtheorem{definition}[]{Definition}
\usepackage[justification=centering]{caption}
\usepackage{enumerate} 
\usepackage{cite}
\usepackage{suffix}
\usepackage{mathtools}
\usepackage{tabularx, booktabs}
\newcolumntype{C}{>{\centering\arraybackslash}X} 
\usepackage[font=small, labelfont=bf]{caption}

\title{Fast, Parameter free Outlier Identification for Robust PCA}
\begin{document}
\author{\IEEEauthorblockN{Vishnu Menon,
		Sheetal Kalyani\\}
	\IEEEauthorblockA{
		Department of Electrical Engineering, Indian Institute of Technology Madras\\
		Chennai, India - 600036\\
		Email: ee16s301@ee.iitm.ac.in,
		skalyani@ee.iitm.ac.in,
		}}
\maketitle
\begin{abstract}
	Robust PCA, the problem of PCA in the presence of outliers has been extensively investigated in the last few years. Here we focus on Robust PCA in the column sparse outlier model. The existing methods for column sparse outlier model assumes either the knowledge of the dimension of the lower dimensional subspace or the fraction of outliers in the system. However in many applications knowledge of these parameters is not available. Motivated by this we propose a parameter free outlier identification method for robust PCA which a) does not require the knowledge of outlier fraction, b) does not require the knowledge of the dimension of the underlying subspace, c) is computationally simple and fast. Further, analytical guarantees are derived for outlier identification and the performance of the algorithm is compared with the existing state of the art methods.
\end{abstract}
\section{Introduction}
Principal Component Analysis (PCA) \cite{jolliffe2002principal} is a very widely used technique in data analysis and dimensionality reduction. Singular Value Decomposition (SVD) of the data matrix $\textbf{M}$ \cite{shlens2014tutorial} is known to be very sensitive to extreme corruptions in the data \cite{candes2011robust}, \cite{xu2010robust}, \cite{rahmani2016coherence} and hence robustifying the PCA process becomes a necessity. Robust PCA is typically an ill posed problem and it is of significant importance in a wide variety of fields like computer vision, machine learning, survey data analysis and so on. Recent survey papers \cite{vaswani2018static}, \cite{lerman2018overview} outline the various existing techniques for robust subspace recovery and robust PCA. Of the numerous approaches to robust PCA over the years \cite{ammann1993robust}, \cite{de2003framework}, one way to model extreme corruptions in the given data matrix $\textbf{M}$, is using the following decomposition \cite{wright2009robust}, \cite{chandrasekaran2011rank}, \cite{zhou2010stable}, \cite{candes2011robust}:
\begin{equation}\label{emodel}
\textbf{M}  = \textbf{L}+\textbf{S}, 
\end{equation}
where $\textbf{S}$ encapsulates all the corruptions and is assumed to be sparse and $\textbf{L}$ is low rank. Thus robust PCA becomes a process of decomposing the given matrix into a low rank matrix plus a sparse matrix. The problem can be formulated as a convex
problem, using techniques of convex relaxation inspired from
compressed sensing \cite{candes2006compressive}, as \cite{wright2009robust}, \cite{candes2011robust}
\begin{equation}\label{ereformul}
\begin{aligned}
& \underset{\textbf{L}, \textbf{S}}{\text{minimize}}
& & \|\textbf{L}\|_*+\lambda \|\textbf{S}\|_1
& \text{s.t} && \textbf{M} = \textbf{L}+\textbf{S}, 
\end{aligned}
\end{equation}
where $ \|\textbf{L}\|_*$ is the nuclear norm computed as the sum of singular values of a matrix and $\|\textbf{S}\|_1$ is the $l_1$ norm of vector formed by vectorizing the matrix. In \cite{candes2011robust}, an optimal value for $\lambda$ was proposed and theoretical guarantees for the exact recovery of the low rank matrix was given assuming the popular uniform sparsity model. To solve (\ref{ereformul}), several algorithms were proposed including  \cite{yi2016fast}, \cite{hsu2011robust}, \cite{chiang2016robust} with the aim of reducing the complexity of the process and improving speed and performance. Non-Convex algorithms have also been proposed for robust PCA\cite{netrapalli2014non}, \cite{kang2015robust} which are significantly faster than convex programs. 

Another popular model, the one that we will adopt in this paper, is the outlier model\footnote{Throughout the paper, the term outlier model indicates the model where each column of $\textbf{M}$ is either an inlier or an outlier}. In this model each column in $\textbf{M}$ is considered as a data point in $\mathbb{R}^n$. The points that lie in a lower dimensional subspace of dimension $r$ are the inliers and others which do not fit in this subspace are the outliers. The model is still (\ref{emodel}), but the matrix $\textbf{S}$ is column sparse. Several methods have been developed over the years, like methods based on influence functions \cite{jolliffe2002principal}, the re weighted least squares method \cite{ma2015generalized}, methods based on random consensus (RANSAC) \cite{fischler1987random}, based on rotational invariant $l_1$ norms \cite{ding2006r} etc for the outlier model. In \cite{xu2010robust}, a convex formulation of the process is given and iterative methods have been proposed to solve it. Also the problem has been extended to identifying outliers when the inlying points come from a union of subspaces as in \cite{soltanolkotabi2012geometric}, \cite{soltanolkotabi2014robust}. Recent works have attempted to develop simple non iterative algorithms for robust PCA with the outlier model  \cite{rahmani2016coherence}. Other methods which aims at solving robust PCA through this model include \cite{lerman2014fast}, \cite{zhang2014novel} and works based on thresholding like \cite{cherapanamjeri2017thresholding}. Most of the algorithms proposed are either iterative and complex and/or would require the knowledge of either the outlier fraction or the dimension of the low rank subspace or would have free parameters that needs to be set according to the data statistics. What we aim to do in this paper is to propose an algorithm for removal of outliers that is computationally simple, non iterative and parameter free. Once the outliers are removed from the data in the outlier model, the classical methods for PCA may be applied for subspace recovery. 
\subsection{Related work}
We briefly describe some of the key literature in the area of robust PCA and highlight how our proposed work differs from and/or is inspired by them. The popular work \cite{candes2011robust}, assuming a uniform sparsity model on the corruptions, solves (\ref{ereformul}) using Augmented Lagrange Multiplier (ALM) \cite{lin2010augmented} which is an iterative process that requires certain parameters to be set. Ours uses an outlier model and hence we cannot compare our method with the work  in \cite{candes2011robust}. In an outlier model, \cite{xu2010robust}, proposes solving the following convex optimization problem for robust PCA:
\begin{equation}
\begin{aligned}
& \underset{\textbf{L}, \textbf{S}}{\text{minimize}}
& & \|\textbf{L}\|_*+\lambda \|\textbf{S}\|_{1, 2}
& \text{s.t} && \textbf{M} = \textbf{L}+\textbf{S}, 
\end{aligned}
\end{equation}
where $\|\textbf{S}\|_{1, 2}$ is the sum of $l_2$ norms of the columns of the matrix. The paper also proposes a value for the parameter, namely $\lambda = \frac{3}{7\sqrt{\gamma N}}$, where $\gamma$ is the fraction of outliers in the system. While \cite{xu2010robust} assumes the knowledge of $\gamma$, in many cases $\gamma$ is typically unknown. Another recent work \cite{cherapanamjeri2017thresholding} that bases its algorithms on thresholding also requires the knowledge of the target rank, i.e the dimension of the subspace. The work in \cite{soltanolkotabi2012geometric} analyzes the removal of outliers from a system where the inliers come from a union of subspaces and involves solving multiple $l_1$ optimization problems. While there exists a lot of existing techniques and algorithms \cite{yang2010fast} for solving the $l_1$ optimization problem, most of them requires certain parameters to be set and are iterative. After solving the optimization problem, a data point is classified as an inlier or outlier using thresholding in \cite{soltanolkotabi2012geometric}. Although the proposed threshold in \cite{soltanolkotabi2012geometric} is independent of the dimension of the subspace $r$ or the number of outliers, the underlying optimization problem is not parameter free and since multiple optimization problems have to be solved, the procedure is also rather complex.

A fast algorithm for robust PCA was recently proposed in \cite{rahmani2016coherence} which involves looking at the coherence of the data points with other points and identifying outliers as those points which have less coherence with the other points. The authors give theoretical guarantees for the working of the algorithm for the outlier model. In the two methods that have been proposed for identifying the true subspace, knowledge of either the number of outliers or the dimension of the underlying subspace is required. The algorithm proposed for outlier removal in our work is in spirit a parameter free extension to the work in \cite{rahmani2016coherence}. 
\subsection{Motivation and proposed approach}
The main motivation behind this work is to build parameter free algorithms for robust PCA. By parameter free we mean an algorithm which does not require the knowledge of parameters such as the dimension of true subspace or the number of outliers in the system nor it has a tuning parameter which has to be tuned according to the data. Tuning parameters in any algorithm present a challenge, as the user then would have to decide either through cross validation  \cite{arlot2010survey} or prior knowledge on how to set them. Recently, there have been attempts to make algorithms parameter free in the paradigm of sparse signal recovery \cite{kallummil2017signal, vats2014path}, \cite{stoica2012spice}, \cite{lederer2015don} and these were shown to have results comparable with the ones when the true parameters such as the sparsity of the signal were known. Motivated by this, in this paper we propose a parameter free algorithm for robust PCA. While there exists a vast literature on robust PCA algorithms, to make parameter free variants of them, one would have come up with novel modifications for each of them separately. 
In this work, we focus on obtaining a computationally efficient parameter free algorithm for outlier removal in robust PCA. The recent work in \cite{rahmani2016coherence} is both simple and non iterative in the sense that it is a one shot process which does not involve an iterative procedure to solve an optimization problem like in \cite{soltanolkotabi2012geometric}. However it is not parameter free. We propose an algorithm for robust PCA in the outlier model, which does not require the knowledge of number of outliers or the dimension of the underlying subspace and our key contributions are:
\begin{itemize}
	\item[i] We will demonstrate through theory and simulations, that the proposed parameter free algorithm identifies all the outliers correctly with high probability in the outlier model, 
	\item[ii] The proposed method is threshold based and we derive a threshold, which does not require the knowledge of the dimension of the underlying subspace and the number of outliers present, for its calculation. This threshold lower bounds the outlier scores with high probability even in presence of Gaussian noise irrespective of the noise variance. 
	\item [iii] Our algorithm guarantees outlier identification, however it does not guarantee that all inliers would be recovered correctly. We derive theoretical conditions under which a good percentage of inliers are recovered and also the conditions when the algorithm does not recover all the inliers points.
	\item [iv] We compare our results with existing algorithms for outlier removal and subspace recovery in terms of subspace recovery error and running time.
\end{itemize}
 Our hope is that this algorithm will serve as a starting point for further progress in parameter free algorithms for robust PCA.
\section{Problem setup and notations}\label{s1}
	We are given $N$ data points, each from an $n$ dimensional space $\mathbb{R}^n$, denoted by $\textbf{m}_i \in \mathbb{R}^n$, arranged in a data matrix $\textbf{M} = [\textbf{m}_1, \textbf{m}_2...\textbf{m}_N] \in \mathbb{R}^{n\text{x}N}$. In this paper we will be working with  $l_2$ normalized data points, namely $\textbf{x}_i = \dfrac{\textbf{m}_i}{\|\textbf{m}_i\|_2}$. Here $ \|.\|_2$, denotes the $l_2$ norm. Let the normalized data matrix be denoted as $\textbf{X} = [\textbf{x}_1, \textbf{x}_2...\textbf{x}_N]$. Let $\mathbb{S}^{n-1}$ denote the unit hypersphere in $\mathbb{R}^n$. Then $\mathbb{S}^{n-1} = \{\textbf{x}\text{	} |\text{	}\textbf{x}\in \mathbb{R}^n, \|\textbf{x}\|_2 = 1\}$, i.e the $l_2$ ball in $\mathbb{R}^n$ and all points in $\textbf{X} \in \mathbb{S}^{n-1}$. We assume that out of the $N$ data points, $(1-\gamma)N$ of them lie in a low dimensional subspace $\mathcal{U}$ of dimension $r$, those we will refer to as inliers and the rest $\gamma N$ points are randomly spread out on the $n$ dimensional space, we will call them outliers. The parameters $\gamma$ which is the ratio of number of outliers to the total number of data points and $r$, dimension of the true subspace, are unknown. Let $\mathcal{I}$ denote the index set of inliers and $\mathcal{O}$ denote the index set of outliers, i.e $\mathcal{I} = \{i\text{		}| \text{		}\textbf{x}_i \text{ is an inlier}\}$ and $\mathcal{O} = \{i\text{		}| \text{		}\textbf{x}_i \text{ is an outlier}\}$. Hence the matrix $\textbf{X}$ can be segregated as $\textbf{X} = [\textbf{X}_\mathcal{I}, \textbf{X}_\mathcal{O}] $, where $\textbf{X}_\mathcal{I}$ are the set of inlier points and  $\textbf{X}_\mathcal{O}$ are the set of outlier points. Note that $|\mathcal{I}| = (1-\gamma)N$ and $|\mathcal{O}| =\gamma N$, where $|.|$ denotes the cardinality of a set. We will denote $N_\mathcal{I} =|\mathcal{I}|$. In this paper the following assumption is made on outliers (same as Assumption 1 in \cite{rahmani2016coherence}).
\begin{Assumption}\label{amain}
The subspace $\mathcal{U}$ is chosen uniformly at random from the set of all $r$ dimensional subspaces and the inlier points are sampled uniformly at random from the intersection of $\mathcal{U}$ and $\mathbb{S}^{n-1}$. The outier points are sampled uniformly at random from $\mathbb{S}^{n-1}$.
\end{Assumption}
 The problem we will be focusing on is to remove the set of outliers from the matrix or to find  $\mathcal{O}$ without the knowledge of both the parameters $\gamma$ and $r$. We first list some essential definitions.
\begin{definition}
Let $\theta_{ij}$ denote the principal angle between two data points $\textbf{x}_i$ and $\textbf{x}_j$, i.e 
\begin{equation}\label{etheta}
\theta_{ij} = cos^{-1}(\textbf{x}_i^T\textbf{x}_j) \hskip70 pt \theta_{ij} \in [0, \pi]
\end{equation}
\end{definition}
\begin{definition}
	The acute angle between two points denoted by $\phi_{ij}$ is defined as 
	\begin{align}\label{ephi}
	\phi_{ij} &= cos^{-1}(|\textbf{x}_i^T\textbf{x}_j|)\\
	&=\begin{cases}
	\theta_{ij} & \text{for } \theta_{ij}\leq \dfrac{\pi}{2}\\
	\\
	\pi-\theta_{ij} & \text{for } \theta_{ij} > \dfrac{\pi}{2}\\
	\end{cases}
	\end{align}
\end{definition}
Clearly $\phi_{ij}\in [0, \frac{\pi}{2}]$. Also $\phi_{ii}= \theta_{ii} =0$. 
\begin{definition}
The minimum angle subtended by a point denoted as $q_i$ is given by 
\begin{align}\label{eq}
q_i = \underset{j=1, ..N, j\neq i}{\min\text{			}}\phi_{ij} \hskip30pt \forall i \in \{1, 2, ...N\}
\end{align}
\end{definition}
We also call $q_i$ as the score\footnote{We will, in this paper use the term "score" of a point to indicate the minimum acute angle from the set of all acute angles subtended by a data point with other points.} of a point indexed by $i$. Let $q_{(i)}$ represent the values of $q_i$ sorted in descending order, i.e $q_{(1)}\geq q_{(2)} \geq... q_{(N)}$. Let $L$ be the index set corresponding to values of $q_i$ after sorting in descending order. From here on for convenience, we will index the sorted scores by $l$, i.e $q_l$ corresponds to the $q_i$ values after sorting in descending order or $q_l \geq q_{l+1}$. We define:
\begin{definition}\label{dqo}
	Denote the minimum outlier score as $q_\mathcal{O}$, then $q_{\mathcal{O}} = \underset{i \in \mathcal{O}}{\min}\text{		}q_i$
\end{definition}
\begin{definition}\label{dqi}
	Denote the maximum inlier score as $q_\mathcal{I}$, then $	q_\mathcal{I}  = \underset{i \in \mathcal{I}}{\max}\text{	} q_i$
\end{definition}
Now we will also define two properties that characterizes an algorithm for outlier removal.
\begin{definition}[Outlier Identification Property, OIP($\alpha$) ]\label{doip}
	An algorithm for outlier removal is said to have Outlier Identification Property OIP($\alpha$), when the outlier index set estimate of the algorithm contains all the true outlier indices i.e $\hat{\mathcal{O}} \supseteq \mathcal{O}$ with a probability at least $1-\alpha$. 
\end{definition}
\begin{definition}[Exact recovery Property, ERP($\alpha$)]\label{derp}
	An algorithm for outlier removal is said to have Exact Recovery Property, ERP($\alpha$) when it recovers all the inlier points or $\hat{\mathcal{I}} = \mathcal{I}$ with a probability at least $1-\alpha$.
\end{definition}
ERP($\alpha$) is a stronger condition than OIP($\alpha$). An algorithm which has ERP($\alpha$) will also have OIP($\alpha$) as in this case, $\hat{\mathcal{O}}=\mathcal{O}$ with a probability at least $1-\alpha$. 

\textbf{Other Notations}: Let $\Gamma(.)$ denote the gamma function. $\mathbb{E}[Y]$ denotes the expectation, $var(Y)$ the variance and $\sigma_Y$ the standard deviation of the random variable $Y$. $\mathcal{N}(\mu, \sigma^2)$ denote a normal distribution with mean $\mu$ and variance $\sigma^2$. Let $F_{\mathcal{N}}(.)$ denote the standard normal cdf, $F_{\mathcal{N}}(y) =\dfrac{1}{\sqrt{2\pi}} \int\limits_{-\infty}^{y}e^{\frac{-x^2}{2}}dx$. 
\section{Algorithm and features}\label{s2}
We will first discuss in brief the coherence pursuit (CoP) algorithm in \cite{rahmani2016coherence}, since our work can be regarded as a parameter free variant of CoP. The basic principle behind CoP algorithm \cite{rahmani2016coherence} is that the inlier points are more coherent amongst themselves and the outliers are less coherent. Hence for each point a metric is computed as the norm (either $l_1$ or $l_2$ norm) of a vector in $\mathbb{R}^{N-1}$ whose components are the coherence values that a point has with all the other data points. The expectation is that once these metrics are sorted in descending order, the inliers come first as the outlier metrics are supposed to be much less compared to the inlier metrics. Then the authors have proposed two schemes to remove the outliers and recover the true underlying subspace. The first scheme tries to remove the outliers and then perform PCA to get the true subspace. Here the outlier removal process assumes the knowledge of the maximum number of outliers in the system. The second scheme is an adaptive column sampling technique that generates an $r$ dimensional subspace from inlier points, with the assumption that the parameter $r$ is known. 

The proposed scheme, works with angles instead of coherence, and the score that we compute is the minimum angle subtended by a point instead of the norm as is done in \cite{rahmani2016coherence}. We develop a high probability lower bound for outlier scores independent of the unknown parameters, which makes our algorithm parameter free.
\begin{figure*}[t]
	\begin{subfigure}[b]{0.24\textwidth}
		\includegraphics[width=\linewidth, height=4cm]{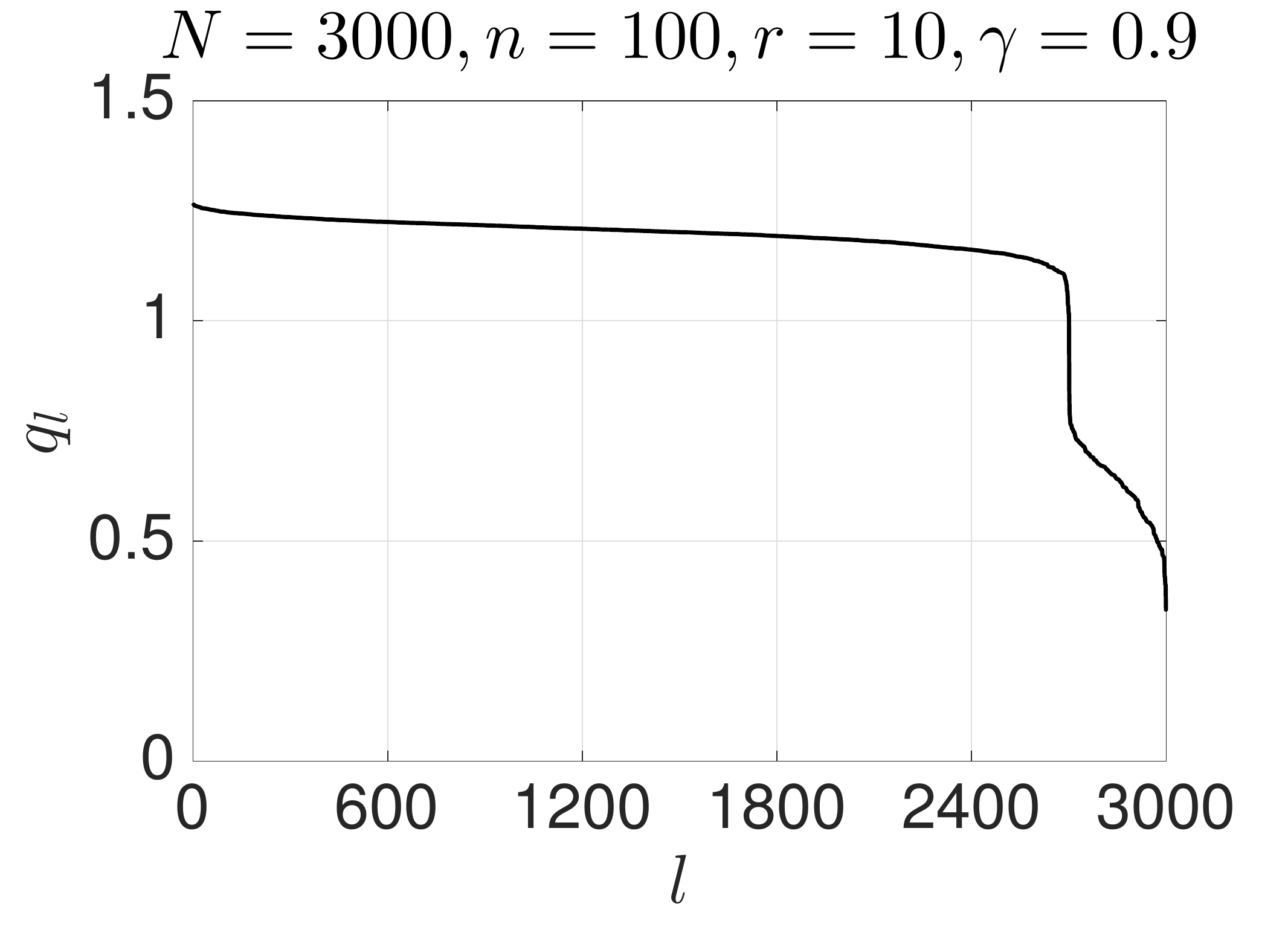}
		\caption{}
		\label{}
	\end{subfigure}
	\hfill
	\begin{subfigure}[b]{0.24\textwidth}
		\includegraphics[width=\linewidth, height=4cm]{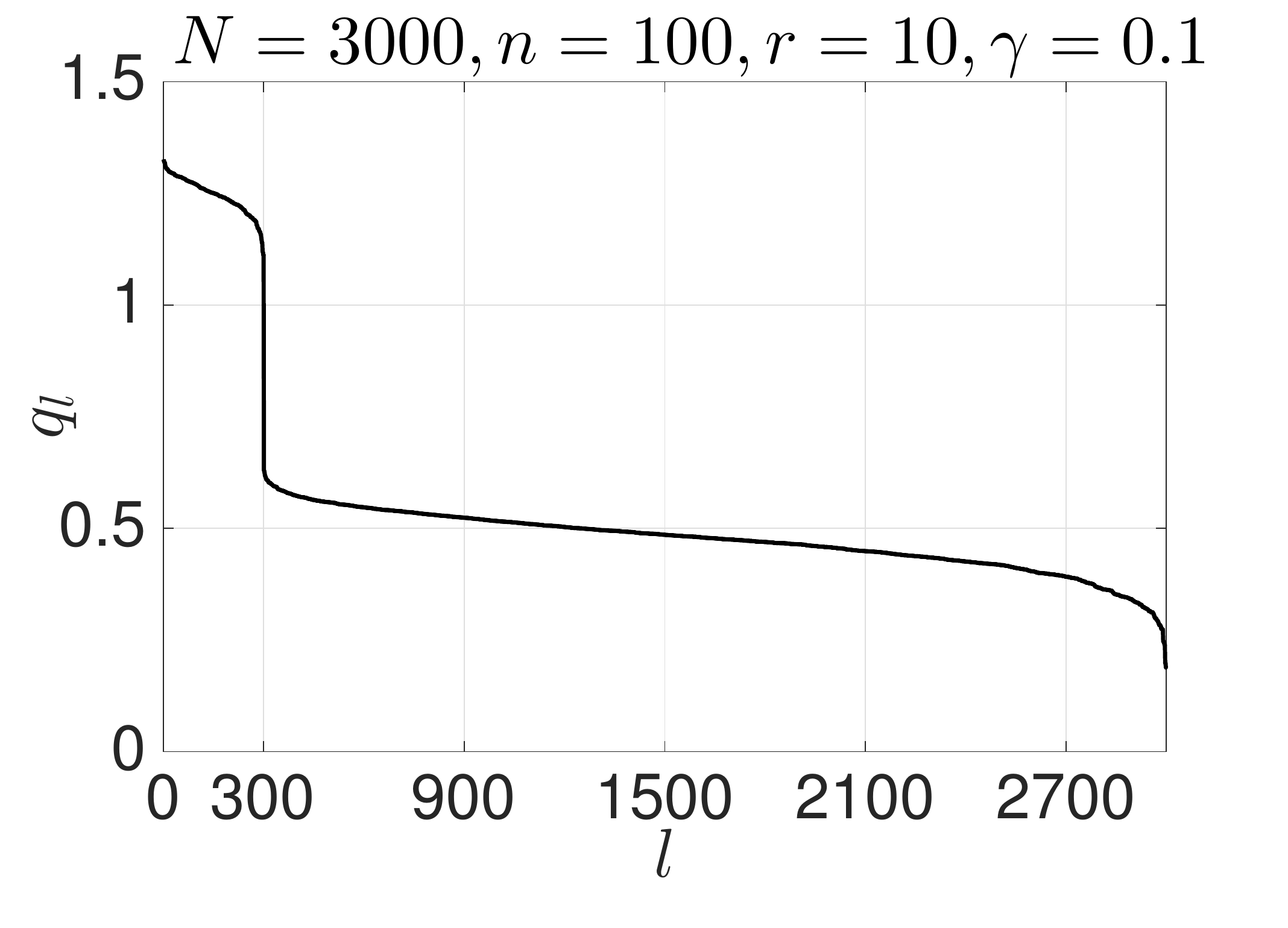}
		\caption{}
		\label{}
	\end{subfigure}
	\hfill
	\begin{subfigure}[b]{0.24\textwidth}
		\includegraphics[width=\linewidth, height=4cm]{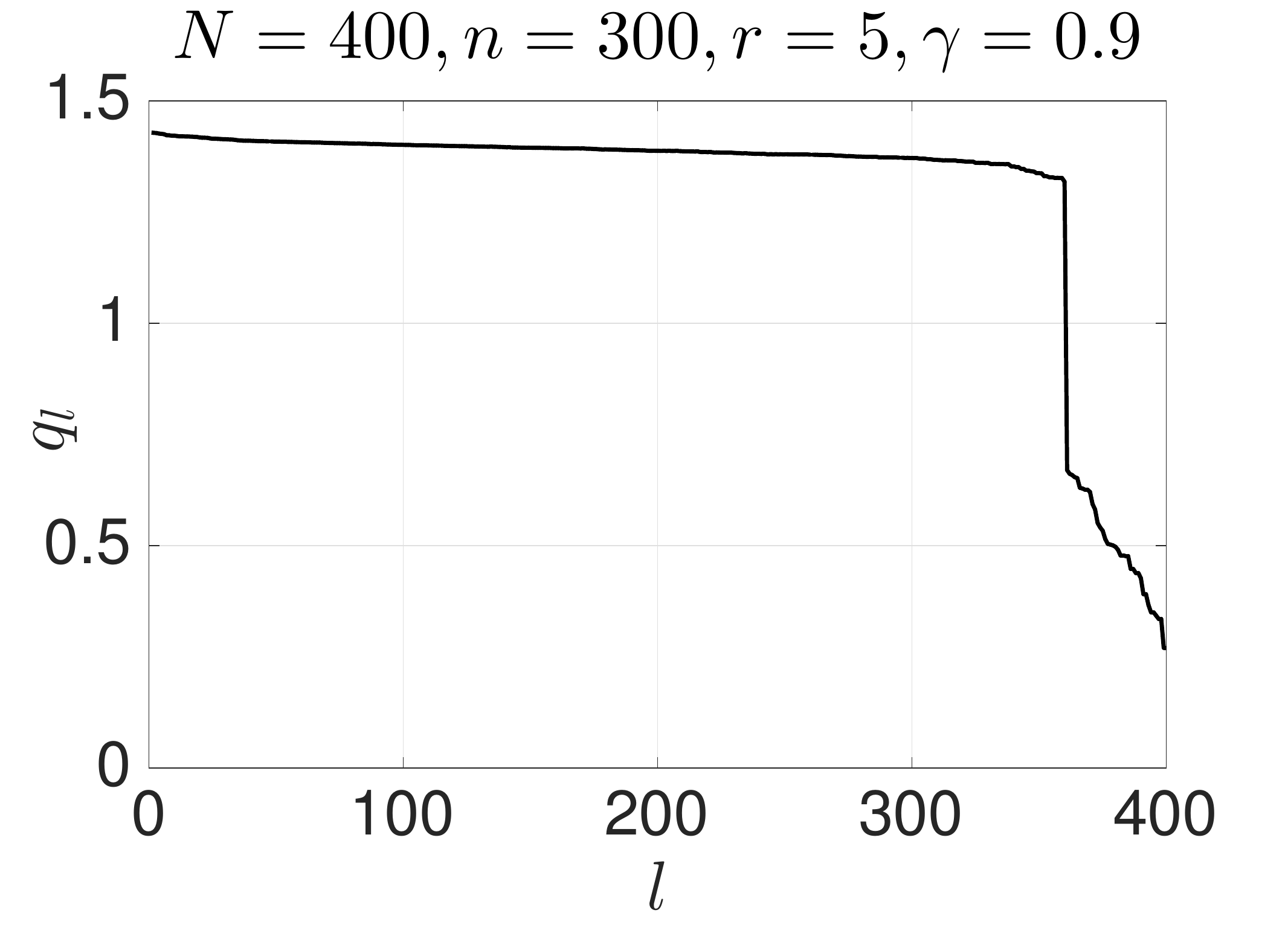}
		\caption{}
		\label{}
	\end{subfigure}
	\hfill
	\begin{subfigure}[b]{0.24\textwidth}
		\includegraphics[width=\linewidth, height=4cm ]{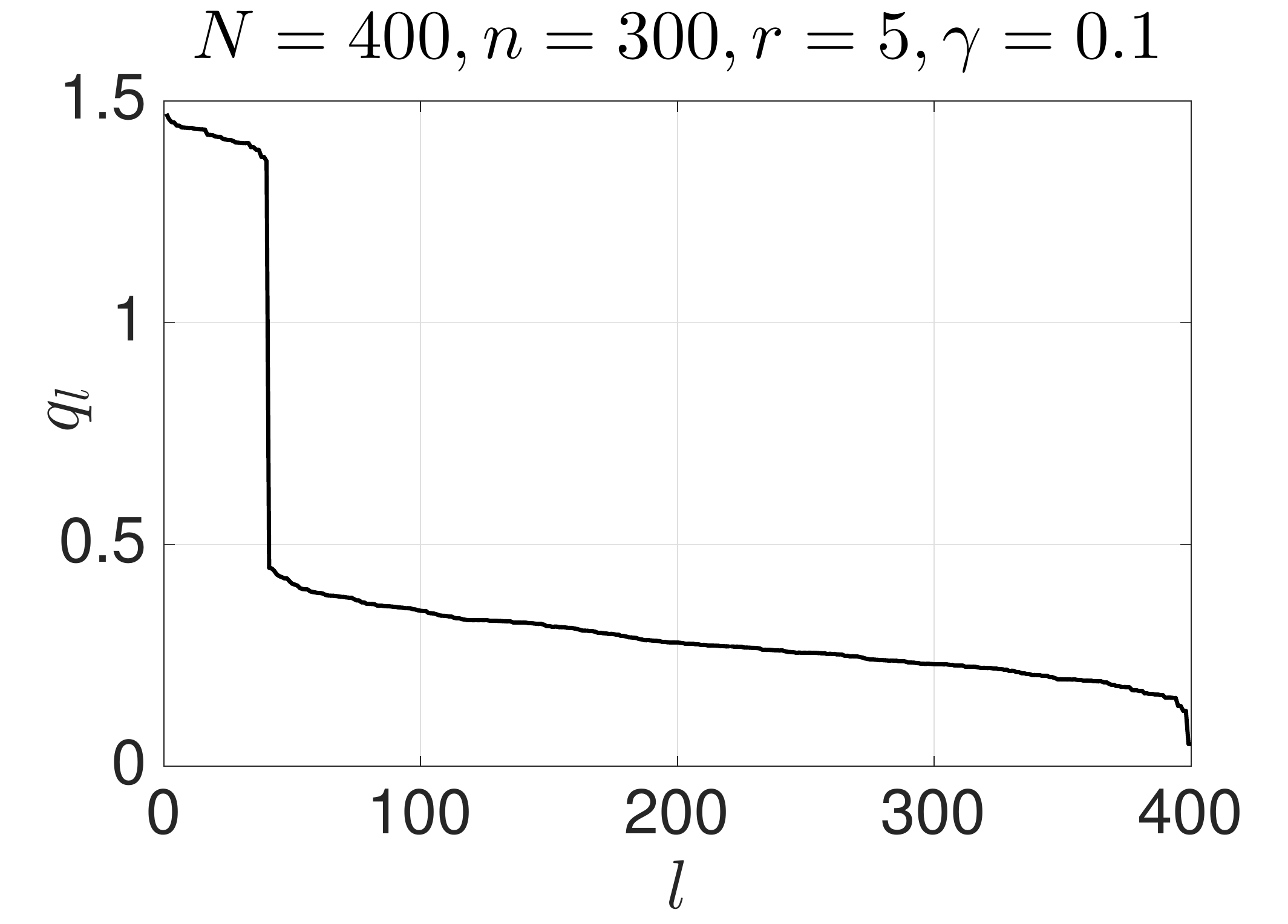}
		\caption{}
		\label{}
	\end{subfigure}
	\hfill
	\caption{Behaviour of $q_l$ for various values of $n$, $N$, $r$ combinations and varying $\gamma$}
	\label{fqsep}
\end{figure*}
 \subsection{Basic Principle and description}
 The folklore ``two high dimensional points are almost always orthogonal to each other'' has been rigorously proved in \cite{cai2013distributions} and this is what we exploit. The proposed algorithm works on the principle that, outlier points subtend larger angles (very close to $\frac{\pi}{2}$) with rest of the points, but inlier points, since they lie in a smaller dimensional subspace, subtend smaller angles with other inlier points and hence would have a much smaller score $q_i$ as compared to an outlier. An example of the clear separation between the $q_l$ values as a function of $l$ can be seen in Fig \ref{fqsep} for various values of $r$, $n, $, $N$, $\gamma$, which shows that the property holds even at low inlier fraction.
 So the expectation is that when the scores $q_i$ have been sorted in descending order, we will have all the outlier indices appearing first and then after a dip in $q_l$, the inlier indices appear. In the proposed method we will exploit this property to develop an algorithm that removes outliers and is also parameter free. In the algorithm we classify a point as an outlier whenever its score $q_i$ is greater than a threshold $\zeta$ given by
 \begin{equation}\label{eth}
 \begin{aligned}
 \zeta = &\Big[\dfrac{4\sqrt{\pi} \times\Gamma(\frac{n+1}{2})log(\frac{1}{1-\frac{\alpha}{2}})}{N^2\Gamma(\frac{n}{2})}\Big]^{\frac{1}{n-1}},
 \end{aligned}
 \end{equation}
$\zeta$ is such that ${P}(q_\mathcal{O}>\zeta) \geq 1-\alpha$, i.e it acts as a lower bound for $q_i$ (See Theorem \ref{tthr} for more details). 
The proposed algorithm is given in tabular form as Algorithm 1.
\begin{algorithm}[h]
	\caption{Removal of Outliers using Minimum Angle (ROMA)}
	\textbf{Input}:Data matrix $\textbf{M}$\\
	\textbf{Procedure:}
	\begin{algorithmic}[1]
		\State Define $\textbf{X}$, with columns $\textbf{x}_i=\frac{\textbf{m}_i}{\|\textbf{m}_i\|_2}$
		\State Calculate $\phi_{ij}$ for $i, j=1, 2..N$ as in (\ref{ephi})
		\State Set $\alpha=.05$
		\State Threshold, $\zeta \gets\Big[\dfrac{4\sqrt{\pi} \times\Gamma(\frac{n+1}{2})log(\frac{1}{1-\frac{\alpha}{2}})}{N^2\Gamma(\frac{n}{2})}\Big]^{\frac{1}{n-1}}$
		\State Calculate $q_i$ for $i=1, 2..N$ as in (\ref{eq})
		\State Sort $q_i$ in descending order and get $q_{l}$ 
		\State $L \gets$ Sorted indices
		\State $l^* \gets \underset{l}{\max}\{l\text{	}|\text{	}q_l >\zeta\}$ 
		\State $\hat{\mathcal{O}} \gets L(1:l^*)$, $\hat{\mathcal{I}} \gets L(l^*+1:N)$
	\end{algorithmic}
	\textbf{Output}: $\hat{\mathcal{I}}, \hat{\mathcal{O}} $ 
\end{algorithm}

The key steps are, 
\par i First the input data matrix is column normalized and the acute angles subtended by each point with other points as in (\ref{ephi}) are calculated for all data points. 
\par ii Then the score for each point $q_i$ is computed by taking the minimum of the angles subtended by that point as in (\ref{eq}). Then the scores are sorted in descending order to get $q_l$.
\par iii The last index $l^*$ until where $q_l$ remains above the threshold  $\zeta$ is noted as the boundary, all the sorted indices before and including this point are classified as outlier indices and the sorted indices beyond this point are classified as inlier indices. 

The algorithm which we will call Removal of Outliers using Minimum angle (ROMA), is a parameter free algorithm and requires as input only the data matrix. 
\subsection{Feature - Parameter free}
The main feature of the algorithm is that it does not have any dependencies on the unknown parameters. As seen clearly, the threshold we have proposed only requires $N$  and $n$ for its computation and is also independent of noise statistics. Once all the outliers have been identified and removed, the clean points can be used to obtain a low rank representation using classical PCA by SVD. For the noiseless case, PCA also does not require the knowledge of any parameter. In the presence of additive Gaussian noise $\textbf{w}_i$'s in the data, i.e when $\textbf{m}_i = \textbf{m}^0_i+\textbf{w}_i$, $\textbf{w}_i \sim \mathcal{N}(0, \sigma_{w}^2\textbf{I}_n)$, there are several methods for selecting the number of principal components after SVD like BIC \cite{rissanen1978modeling}, geometric AIC \cite{kanatani1998geometric}, and other recent methods proposed in \cite{choi2017selecting}, \cite{donoho2013optimal}. 
\subsection{Feature - Simplicity}
ROMA is a simple to implement algorithm and the main complexity lies in computing all the angles subtended. This requires computation of $N(N-1)$ angles and hence the complexity is $O(N^2)$. It is computationally similar to CoP\cite{rahmani2016coherence} for its step of outlier identification and hence has the same complexity. The second step to implement robust PCA, would be an SVD on the inlier points recovered by the algorithm, which if implemented without truncation, has a time complexity of $O(N^3)$\cite{golub2012matrix}. The algorithm does not involve solving a complex optimization problem and is not iterative. The running time comparisons are given in section \ref{snumsim}.
\section{Theoretical Analysis of ROMA}\label{s3}
In this section we address the following points:
\begin{itemize}
	\item[i] Give the reasoning behind why we choose the metric as the minimum acute angle to identify outliers
	\item[ii] Derive the high probability lower bound $\zeta$ on the outlier scores which is parameter free.
	\item[iii] While the lower bound is derived for a noise free scenario, we will show that it is able to handle Gaussian noise as well.
	\item[iv] Derive the conditions when the algorithm follows OIP($\alpha$) and when it does not follow ERP($\alpha$) given in Definitions \ref{doip} and \ref{derp}.
\end{itemize}

As a starting point, we will state two lemmas that describe the distribution of the principal angles $\theta_{ij}$'s made by the points. This involves a slight modification of Lemma 12 in \cite{cai2013distributions}, to distinguish the angles formed by an inlier and outlier.
\begin{lemma}\label{lthetao}
	When $i \in \mathcal{O}$, $\theta_{ij}$'s are identically distributed and it's probability density function is given by
	\begin{equation}
				h(\theta) = \dfrac{1}{\sqrt{\pi}}\dfrac{\Gamma(\frac{n}{2})}{\Gamma(\frac{n-1}{2})} (sin\theta)^{n-2} \hskip20pt \theta \in [0, \pi]
	\end{equation}
	It has an expected value of $\dfrac{\pi}{2}$. Also $h(\theta)$ is well approximated by a Gaussian pdf with mean $\frac{\pi}{2}$ and variance $\frac{1}{n-2}$.
\end{lemma}
\begin{proof}
	Please refer to appendix \ref{app1}
\end{proof}

\begin{lemma}\label{lthetai}
	When $i, j \in \mathcal{I}$, $\theta_{ij}$'s are identically distributed and it's probability density function is given by
	\begin{equation}
	h(\theta) = \dfrac{1}{\sqrt{\pi}}\dfrac{\Gamma(\frac{r}{2})}{\Gamma(\frac{r-1}{2})} (sin\theta)^{r-2} \hskip20pt \theta \in [0, \pi]
	\end{equation}
	It has an expected value of $\dfrac{\pi}{2}$. Also $h(\theta)$ is well approximated by a Gaussian pdf with mean $\frac{\pi}{2}$ and variance $\frac{1}{r-2}$ whenever $r \geq 5$.
\end{lemma}
\begin{proof}
Please refer to appendix \ref{app1}
\end{proof}
For the algorithm we use the acute angles $\phi_{ij}$'s instead of $\theta_{ij}$'s. Next we discuss the mathematical reason for it.
\subsection{Why use $\phi_{ij}$ instead of $\theta_{ij}$ and lower bound the outlier score?}
From Lemmas \ref{lthetao} and \ref{lthetai}, for an outlier point, the principal angle it makes with any other point, be it an inlier or outlier is typically concentrated around $\frac{\pi}{2}$ especially at large $n$. On the other hand, because the dimension $r$ of the subspace $\mathcal{U}$ is much smaller than $n$, the angle that an inlier makes with another inlier is more spread around $\frac{\pi}{2}$. If one wanted to classify a point indexed by $i$, as an inlier or outlier based on the minimum principal angle that it subtends, the following strategy may be employed. The point indexed by $i$ is an outlier when $\beta_l \leq \underset{j=1, 2...N, j\neq i}{\min}\theta_{ij} \leq \beta_u$, with $\beta_l$ and $\beta_u$ being very close to $\frac{\pi}{2}$ and it is an inlier when $\underset{j=1, 2...N, j\neq i}{\min}\theta_{ij}<\beta_l$ or when $\underset{j=1, 2...N, j\neq i}{\min}\theta_{ij}>\beta_u$. This would require looking into different classification regions and two thresholds, which can be avoided by using the acute angle $\phi_{ij}$. In this case the classification requires just one threshold, say $\beta_l$. When using $\phi_{ij}$ whose properties have been characterized in appendix \ref{app3}, the minimum acute angle that an outlier makes becomes very close to $\frac{\pi}{2}$ and a point may be classified as an outlier when $\underset{j=1, 2...N, j\neq i}{\min}\phi_{ij}\geq \beta_l$. For an inlier, it will be much smaller than $\frac{\pi}{2}$, and a point may be classified as an inlier, when $\underset{j=1, 2...N, j\neq i}{\min}\phi_{ij}< \beta_l$. Hence the problem of outlier identification reduces to finding one appropriate threshold to be applied on $q_i$ defined in (\ref{eq}). If we can derive a high probability lower bound $\zeta$ on $q_i$, $i \in \mathcal{O}$,  such that $\forall i \in \mathcal{O}$, $\mathbb{P}(q_i >\zeta) \geq 1-\alpha$, then we can use this as the threshold to classify a point as an inlier or outlier. Further for the algorithm to be parameter free, we derive $\zeta$ such that it only depends on the number of data points $N$ and the ambient dimension $n$, which are of course always known. The next subsection gives the derivation of the bound $\zeta$.
\subsection{Parameter free threshold lower bounding $q_i$, $i \in \mathcal{O}$}\label{ssmain}
Below theorem gives the lower bound $\zeta$ on outlier scores:
\begin{theorem}\label{tthr}
	Under Assumption \ref{amain}, the minimum acute angle subtended by an outlier point is greater than $\zeta = \Big[\dfrac{4\sqrt{\pi} \times\Gamma(\frac{n+1}{2})log(\frac{1}{1-\frac{\alpha}{2}})}{N^2\Gamma(\frac{n}{2})}\Big]^{\frac{1}{n-1}}$, with high probability of at least $1-\alpha$, i.e 
	\begin{equation}\label{ebnd}
	\begin{aligned}
	q_i > \Big[\dfrac{4\sqrt{\pi} \times\Gamma(\frac{n+1}{2})log(\frac{1}{1-\frac{\alpha}{2}})}{N^2\Gamma(\frac{n}{2})}\Big]^{\frac{1}{n-1}} &\hskip30pt \forall i \in \mathcal{O}\\
	&\text{w.p}>1-\alpha
	\end{aligned}
	\end{equation}
\end{theorem}
\begin{proof}
	We use Definition \ref{dqo} for $q_\mathcal{O}$. If $q_\mathcal{O}>\zeta$, then $q_i >\zeta$, $\forall i \in \mathcal{O}$. Hence we will work with $q_\mathcal{O} =\underset{i \in \mathcal{O}}{\min}\text{			} \underset{j=1, 2, ..N, j\neq i}{\min}\text{			}\phi_{ij}, $. Define, 
	\begin{equation}\label{eminangle}
	\theta_{min}^{\mathcal{O}} = \underset{i \in \mathcal{O}}{\min}\text{			} \underset{j=1, 2, ..N, j\neq i}{\min}\text{			}\theta_{ij}, 
	\end{equation} 
	which is the minimum principal angle that an outlier point makes. Further define, 
	\begin{equation}\label{emaxangle}
	\theta_{max}^{\mathcal{O}} = \underset{i \in \mathcal{O}}{\max}\text{			} \underset{j=1, 2, ..N, j\neq i}{\max}\text{			}\theta_{ij}, 
	\end{equation} 
	which is the maximum principal angle that an outlier point makes. 
	Note that $q_\mathcal{O}$ uses acute angles formed by a point for its computation. We will see, how using principal angles we arrive at a bound for $q_\mathcal{O}$.
	For this, we use the result given in Lemma \ref{ltphi} in appendix \ref{app3} that $\underset{i, j \in \mathcal{J}, i\neq j}{\min} \text{		}\phi_{ij} \leq_{st} \underset{i, j \in\mathcal{J}, i\neq j}{\min} \text{		}\theta_{ij}$\footnote{ where $\leq_{st}$ denote stochastic ordering between random variables. $Y \leq_{st} Z$, when $\mathbb{P}(Y>x) \leq \mathbb{P}(Z>x)$}, for any index set $\mathcal{J}$. From equation (\ref{eangle}) part (c), in Lemma \ref{ltphi},
	\begin{equation}\label{eprobnd}
	\begin{aligned}
	\mathbb{P}(q_\mathcal{O}>x)\geq \mathbb{P}(\theta_{min}^{\mathcal{O}} > x)+ \mathbb{P}(\pi - \theta_{max}^{\mathcal{O}}>x)-1 
	\end{aligned}	
	\end{equation}
	Therefore to obtain a lower bound on $q_\mathcal{O}$, one proceeds by evaluating $\mathbb{P}(\theta_{min}^{\mathcal{O}} > x)$ and $\mathbb{P}(\pi - \theta_{max}^{\mathcal{O}}>x)$. Note, \cite{cai2012phase} shows that the angles formed by uniformly chosen random points in space are only pairwise independent and not mutually independent and hence characterizing the exact distribution of $\theta_{min}^{\mathcal{O}}$ is very difficult. Hence we attempt to bound  $\mathbb{P}(\theta_{min}^{\mathcal{O}} > x)$ and $\mathbb{P}(\pi - \theta_{max}^{\mathcal{O}}>x)$.
	
	For a moment let us deviate from our problem setting and look at a setting where there are $p$ points chosen uniformly at random from $\mathbb{S}^{n-1}$. Lets denote this virtual setting as $\mathcal{V}_p$. Let $\theta_{min}^p$ denote the minimum pairwise principal angle amongst them i.e $\theta_{min}^p = \underset{1\leq i < j \leq p}{\min}\text{	}\theta_{ij}$	and let $\theta_{max}^p  = \underset{1\leq i < j \leq p}{\max}\text{	}\theta_{ij}$. The asymptotic distributions of $\theta_{min}^p$ and $\pi - \theta_{max}^p$ are given by Lemma \ref{lminangle} in appendix \ref{app1}. Applying this lemma, 
	\begin{equation}\label{ecdf}
	\begin{aligned}
	\mathbb{P}(\theta_{min}^p\leq x) &=  \mathbb{P}(\pi - \theta_{max}^p\leq x) = 1-e^{-K p^2 x^{n-1}}\\& \qquad \qquad\qquad\text{for $0 \leq x \leq \pi$}\text{, as } p \to \infty, 
	\end{aligned}
	\end{equation}
	where $K = \frac{1}{\sqrt{4\pi}}\frac{\Gamma(\frac{n}{2})}{\Gamma(\frac{n+1}{2})}$.
	We can state the following remark about the case when $p$ is finite.
	\begin{remark}\label{rfinite}
		For $p$ points chosen uniformly at random from $\mathbb{S}^{n-1}$, and for $\theta_{min}^p$ and $\theta_{max}^p$ as defined before, for $0\leq x\leq \pi$, 
		 \begin{equation}\label{ethtaminp}
		 \begin{aligned}
		 \mathbb{P}(\theta_{min}^p\leq x) &\leq 1-e^{-K p^2 x^{n-1}}\\
		 \mathbb{P}(\pi - \theta_{max}^p\leq x) &\leq 1-e^{-K p^2 x^{n-1}}
		 \end{aligned}
		\end{equation}
	\end{remark}
	This can be argued as follows. Lets start with $p$ points chosen uniformly at random from $\mathbb{S}^{n-1}$. Let the index set be denoted as $\mathcal{J}_p = \{1, 2, ..p\}$. Then we start sampling more and more points such that the number of points grows to infinity, i.e $p \to \infty$, lets denote the index set in such a situation as $\mathcal{J}_{\infty} = \{1, 2, ..p, ...\}$. Clearly $\mathcal{J}_{\infty} \supset \mathcal{J}_{p} $, $\Rightarrow \underset{i, j \in \mathcal{J}_\infty, i\neq j}{\min}\text{	} \theta_{ij} \leq \underset{i, j \in \mathcal{J}_p, i\neq j}{\min}\text{	} \theta_{ij}$. In other words whenever, $\underset{i, j \in \mathcal{J}_p, i\neq j}{\min}\text{	} \theta_{ij}\leq x$, then $\underset{i, j \in \mathcal{J}_\infty, i\neq j}{\min}\text{	} \theta_{ij} \leq x$, while the reverse is not true. Hence,
	\begin{align*}
	\mathbb{P}(\underset{i, j \in \mathcal{J}_p, i\neq j}{\min}\text{	} \theta_{ij} \leq x) &\leq \mathbb{P}(\underset{i, j \in \mathcal{J}_\infty, i\neq j}{\min}\text{	} \theta_{ij} \leq x)
	\end{align*}
	The LHS from above is $\mathbb{P}(\theta_{min}^p\leq x)$ and the RHS in nothing but $\mathbb{P}(\theta_{min}^p\leq x)$ as $p \to \infty$. Hence using (\ref{ecdf}), for $0 \leq x \leq \pi$, 
	 \begin{equation*}
	 \mathbb{P}(\theta_{min}^p\leq x) \leq 1-e^{-K p^2 x^{n-1}}\hskip30pt \text{for finite $p$}.
	 \end{equation*}
	  The same argument can be extended to $\pi - \theta_{max}^p$, to obtain
	  	\begin{equation*}
	  \mathbb{P}(\pi - \theta_{max}^p\leq x) \leq 1-e^{-K p^2 x^{n-1}}\hskip15pt \text{for finite $p$}.
	  \end{equation*}  
	  Let $\mathcal{O}_p \subseteq \{1, 2, ...p\}$ be an index set, and let
	  \begin{equation}\label{eminovirtual}
	  \begin{aligned}
	  \theta_{min}^{\mathcal{O}_p} &= \underset{i \in \mathcal{O}_p}{\min}\text{			}\underset{j \in \{1, 2...p\}, j \neq i}{\min}\text{			}\theta_{ij} \\ 
	  \theta_{max}^{\mathcal{O}_p} &= \underset{i \in \mathcal{O}_p}{\max}\text{			}\underset{j \in \{1, 2...p\}, j \neq i}{\max}\text{			}\theta_{ij}
	  \end{aligned}
	  \end{equation}
	   We know that $\theta_{min}^p$ may be written as
	  \begin{equation*}
	  \theta_{min}^p = \underset{i \in \{1, 2...p\}}{\min}\text{			}\underset{j \in \{1, 2...p\}, j\neq i}{\min}\text{			}\theta_{ij}
	  \end{equation*}
	  Since $\mathcal{O}_p \subset \{1, 2...p\}$, $\theta_{min}^{p} \leq \theta_{min}^{\mathcal{O}_p} $. In other words, $\theta_{min}^{\mathcal{O}_p} \leq x$ $\Rightarrow \theta_{min}^p \leq x$, for $0 \leq x \leq \pi$, while the reverse is not true. This is logical as the minimum of a set is always lesser than the minimum of any subset. Hence 
	  \begin{align*}
	  \mathbb{P}(\theta_{min}^{\mathcal{O}_p} \leq x) &\leq \mathbb{P}(\theta_{min}^p \leq x)\hskip20pt\text{$0 \leq x \leq \pi$.}\\
	  \Rightarrow \mathbb{P}(\theta_{min}^{\mathcal{O}_p} \leq x) & \leq 1-e^{-K p^2 x^{n-1}} \hskip20pt\text{from (\ref{ethtaminp})}
	  \end{align*}
	  The same arguments can be extended to $\pi - \theta_{max}^{\mathcal{O}_p}$ as well. In this case, $\theta_{max}^{\mathcal{O}_p} \geq x$ $\Rightarrow \theta_{max}^p \geq x$ and hence for $0 \leq x \leq \pi$, $\pi-\theta_{max}^{\mathcal{O}_p} \leq x$ $\Rightarrow \pi-\theta_{max}^p \leq x$ and the rest follows. Hence we have for a virtual setting $\mathcal{V}_p$ for any $\mathcal{O}_p \subseteq \{1, 2..p\}$, and $\theta_{min}^{\mathcal{O}_p}$ and $\theta_{max}^{\mathcal{O}_p}$ defined as in (\ref{eminovirtual}), 
	  \begin{equation}\label{emainvirtual}
	  \begin{aligned}
	  \mathbb{P}(\theta_{min}^{\mathcal{O}_p} \leq x) & \leq 1-e^{-K p^2 x^{n-1}}\hskip20pt\text{$0 \leq x \leq \pi$}\\
	  \mathbb{P}(\pi - \theta_{max}^{\mathcal{O}_p} \leq x) & \leq 1-e^{-K p^2 x^{n-1}}\hskip20pt\text{$0 \leq x \leq \pi$.}\\
	  \end{aligned}
	  \end{equation}
	  
	Now let us look at our problem setting, where the outlier points are sampled uniformly at random from $\mathbb{S}^{n-1}$ and the other portion of $N$ points, the inliers are uniformly sampled at random from the intersection of a lower dimensional subspace $\mathcal{U}$ and $\mathbb{S}^{n-1}$ which makes it different from a virtual setting with $N$ points in $\mathbb{S}^{n-1}$, i.e $\mathcal{V}_N$\footnote{Here we denote the number of points in the virtual setting as $N$ instead of $p$, hence $\mathcal{V}_N.$}. We state the following remark.
	\begin{remark}\label{rabouout}
		Under Assumption \ref{amain}, an angle that an outlier point make with any other point have the same statistical properties in the inlier-outlier setting and $\mathcal{V}_N$.
	\end{remark}
	 This can be argued as follows. Using Assumption 1, we know that the inlier points come from a randomly chosen subspace $\mathcal{U}$ and also they are sampled uniformly at random in the subspace and so the angle that an outlier point makes with the inlier points have the same statistical behavior as the angle made with a point uniformly sampled at random from $\mathbb{S}^{n-1}$. Angle that an outlier point make with another outlier already conforms with $\mathcal{V}_N$. Hence the remark is true. Note that this is not true for the angles formed by any two inlier points.

	Therefore result (\ref{emainvirtual}) can be directly applied to $\theta_{min}^{\mathcal{O}}$ and $\theta_{max}^{\mathcal{O}}$, i.e
	\begin{equation}\label{eominfinal}
	\begin{aligned}
	\mathbb{P}(\theta_{min}^{\mathcal{O}} \leq x) &\leq  1-e^{-K N^2 x^{n-1}}\\
	\Rightarrow \mathbb{P}(\theta_{min}^{\mathcal{O}} > x) &\geq e^{-K N^2 x^{n-1}}\\ 
	\text{and	}\mathbb{P}(\pi - \theta_{max}^{\mathcal{O}} > x) &\geq e^{-K N^2 x^{n-1}}
	\end{aligned}
	\end{equation}
	Applying result (\ref{eominfinal}) in (\ref{eprobnd}), we obtain
	\begin{align*}
	\mathbb{P}(q_\mathcal{O}>x)\geq 2e^{-K N^2 x^{n-1}}-1 
	\end{align*}
	
	We want a lower bound $\zeta$ such that $q_\mathcal{O} > \zeta$ with probability at least $1-\alpha$. From the above equation, setting  $2e^{-K N^2x^{n-1}} -1 = 1-\alpha$ and using $\zeta = x$, 
	\begin{equation*}
	\begin{aligned}
	2e^{-K\zeta^{n-1}N^2}-1 &= 1-\alpha
	\Rightarrow  e^{-K N^2 \zeta^{n-1}} = 1- \frac{\alpha}{2}\\
	\Rightarrow K \zeta^{n-1} N^2 &= log(\frac{1}{1- \frac{\alpha}{2}})\\
	\end{aligned}
	\end{equation*}
	Using the value for $K$ from (\ref{eK}), we arrive at the threshold as in equation (\ref{eth})
	\begin{align*}
	\zeta &=  \Big[\dfrac{4\sqrt{\pi} \times\Gamma(\frac{n+1}{2})log(\frac{1}{1-\frac{\alpha}{2}})}{N^2\Gamma(\frac{n}{2})}\Big]^{\frac{1}{n-1}}
	\end{align*}
	Hence $q_\mathcal{O}> \zeta$ with probability at least $1-\alpha$. 
\end{proof}
  Summarizing, we have derived a high probability bound for $q_i$, $i \in \mathcal{O}$, which does not depend on the unknown parameters $\gamma$ and $r$. We have already seen how $\zeta$ is used in Algorithm 1. 
\begin{corollary}\label{cguara}
	The outlier index set estimate output of ROMA, $\hat{\mathcal{O}}$ contains $\mathcal{O}$ with probability at least $1-\alpha$, i.e 
	\begin{equation}
	\mathcal{O} \subseteq \hat{\mathcal{O}} \hskip30pt w.p \geq 1-\alpha
	\end{equation}
\end{corollary}
\begin{proof}
	The proof is a simple application of Theorem \ref{tthr}. After ordering $q_i$ to the ordered scores $q_l$, the algorithm chooses $l^*$ such that $l^*$ is the last point when the scores are above $\zeta$. Hence with probability $1-\alpha$, the sorted indices beyond $l^*$ are not outlier points. Hence, if $L$ denotes the sorted indices, $\hat{\mathcal{I}} = L(l^*+1:N) \subset \mathcal{I}$ with probability $1-\alpha$  or the respective complements, $\mathcal{O} \subseteq \hat{\mathcal{O}}$ with the same probability.
\end{proof}
We have given the guarantee that $\mathcal{O} \subseteq \hat{\mathcal{O}}$ with high probability. The next point of interest would be to see when the identification is exact, i.e $\mathcal{O} = \hat{\mathcal{O}}$ and $\hat{\mathcal{I}} = \mathcal{I}$. 
\subsection{Algorithm properties and guarantees for exact recovery}\label{sguaran}
In this section, we will be looking at the properties of algorithm for exactly recovering the true inlier set. We will state two remarks on ROMA and the properties given in Definitions \ref{doip} and \ref{derp}. 
\begin{remark}\label{roip}
	 When $\zeta$ is as given by (\ref{eth}), ROMA has OIP($\alpha$). And this is regardless of the number of outliers in the system or the dimension of the underlying subspace. 
\end{remark}
\begin{remark}\label{rnoerp}
	The algorithm ROMA has ERP($\alpha$) whenever $q_\mathcal{I} \leq \zeta$ with probability at least $1-\alpha$ and it does not have ERP($\alpha$) whenever $q_\mathcal{I} >\zeta$ with probability at least $\alpha$.
\end{remark}
	When $q_\mathcal{I} \leq \zeta$ with probability at least $1-\alpha$, then $\hat{\mathcal{I}} = \mathcal{I}$ with a probability at least $1-\alpha $ and so the algorithm will have ERP($\alpha$). 
Next, we provide conditions on $n$ and $r$ that ensures that on an average, the inlier and outlier scores are well separated. When these conditions are satisfied, the algorithm recovers a good percentage of inliers, so that the subspace can be recovered efficiently. The smaller the rank of the true subspace, the better the results will be in terms on inlier recovery.
\begin{lemma}\label{lrough}
	$q_i$, $i\in\mathcal{O}$ and $q_i$, $i \in \mathcal{I}$ are well separated when, $r < 2 + \dfrac{2(n-2)}{9\pi}$
	\begin{proof}
		Please refer to appendix \ref{app2}
	\end{proof}
\end{lemma}	

In the next theorem we will derive the theoretical condition when ROMA is guaranteed not to have ERP($\alpha$).
\begin{theorem}\label{trevcond}
	The algorithm ROMA is guaranteed not to have ERP($\alpha$), when $N_\mathcal{I} < 1+\dfrac{1-\alpha}{F_{\phi^{\mathcal{I}}}(\zeta)} - \dfrac{\gamma N F_{\phi^{\mathcal{O}}}(\zeta)}{F_{\phi^{\mathcal{I}}}(\zeta)}  $, where $F_{\phi^{\mathcal{I}}}(.)$ is the cdf function of $\phi_{ij}$, when $i, j \in \mathcal{I}$ and $F_{\phi^{\mathcal{O}}}(.)$ is the cdf function of $\phi_{ij}$, when $i \in \mathcal{I}, j \in \mathcal{O}$
\end{theorem}
\begin{proof}
	Please refer to appendix \ref{app2}
\end{proof}
This theorem gives us conditions when the algorithm is guaranteed to not have ERP($\alpha$). In this case the outlier index estimate, $\hat{\mathcal{O}} \supset \mathcal{O}$ and the points classified as outliers is sure to have had inliers as well, i.e $\hat{\mathcal{O}} \cap \mathcal{I} \neq \Phi$. This usually occurs in cases when the outlier fraction is high. Lets fix the alpha value as $\alpha =0.05$. As an example, let us take the case when $n=60, r=10$ and suppose we have $N=400$ samples, then if we have less than $60$ inlier samples, the algorithm is guaranteed to not have ERP($\alpha$). We will take another example where $n>>r$, say $n=100$ and $r=6$ with the same number of samples $N=400$. In this case only when the number of inlier samples goes below $12$ that the algorithm is guaranteed to not have ERP($\alpha$). In all these cases the algorithm still has OIP($\alpha$). All the prior art also derives similar conditions on performance guarantee, for instance coherence pursuit\cite{rahmani2016coherence} guarantees subspace recovery when the inlier density $\frac{N_\mathcal{I}}{r}$ is sufficiently larger that the outlier density. i.e $\frac{N-N_\mathcal{I}}{n}$ while outlier pursuit \cite{xu2010robust} gives conditions on $\gamma$ and $r$ for successful subspace recovery. The theorem does not state the conditions in which the algorithm is guaranteed to have ERP($\alpha$), it merely gives us extreme cases where it is not. The following lemma should give us an idea about the ERP($\alpha$) of ROMA.
\begin{lemma}\label{lerpb}
	Let $N_{\mathcal{I}}=(1-\gamma)N$ denote the number of inliers and let $q_\mathcal{I}$ be given as in Definition \ref{dqi}. Then we have for $k \in \mathcal{I}$, $\mathbb{P}(q_\mathcal{I} \leq \zeta) \geq 1- N_{\mathcal{I}}\mathbb{P}({\underset{j\in \{1, 2..N\}, j\neq k}{\min} \theta_{kj} > \zeta})$. Hence the algorithm has ERP($N_{\mathcal{I}}\mathbb{P}({\underset{j\in \{1, 2..N\}, j\neq k}{\min} \theta_{kj} > \zeta})$).
\end{lemma}
\begin{proof}
	Please refer appendix \ref{app2}  
\end{proof}
For proceeding further it is required to characterize the complementary cdf (ccdf) of $\underset{j\in \{1, 2..N\}, j\neq k}{\min} \theta_{kj} $ for $k \in \mathcal{I}$. Since the $\theta_{kj}$'s are only pairwise independent and not mutually independent as noted in \cite{cai2013distributions} and \cite{cai2012phase}, finding this ccdf analytically is mathematically very difficult. However one can find $\mathbb{P}({\underset{j\in \{1, 2..N\}, j\neq k}{\min} \theta_{kj} > \zeta})$ empirically through simulations to obtain more insight about inlier recovery properties of ROMA. 
\subsection{Impact of noise on the algorithm}
\begin{remark}
	The algorithm ROMA, retains OIP($\alpha$) even in presence of Gaussian noise irrespective of noise variance. 
\end{remark}   
When Gaussian noise is added to an outlier data point, and the noisy outlier is normalized, it is just like selecting it at random from an $n$ dimensional hypersphere. Hence all the theory and bounds on the outlier score will not change. Noise will however affect the inlier identification of the algorithm as noise is bound to increase the statistic $q_\mathcal{I}$. This means that more inlier scores would be above the threshold and so would be classified as outliers. So in presence of noise, in order for the algorithm to have ERP($\alpha$), it would require more number of samples compared to the number that would have been needed had there been no noise, as expected. Mathematically quantifying the impact of noise on subspace recovery for this algorithm will be a part of future research.
\subsection{Parameter free - What about $\alpha$?}\label{salp}
The algorithm ROMA is parameter free since it does not require any knowledge of $r$ or $\gamma$ for its working. Note that $\alpha$ is merely a term that signifies the probability of success of identification of outliers. When $\alpha$ is made smaller, OIP is satisfied by the algorithm with increasing probability but at the same time the value of the threshold $\zeta$ reduces. This in turn means more inliers could be classified as outliers, which impacts the inlier recovery property of the algorithm. Note $\alpha$ neither depend on data statistics and nor needs to be set by cross validation. If one knows that an application is very sensitive to outliers, $\alpha$ can be set to a very low value. Setting $\alpha = 0.05$ would ensure high success of identification of outliers and a reasonable range on the number of inliers where the algorithm has ERP($\alpha$). Unless otherwise specified, we have set $\alpha = 0.05$ in all our simulations.
\begin{figure*}[h]
	\begin{subfigure}[b]{0.24\textwidth}
		\includegraphics[width=\linewidth, height=6cm]{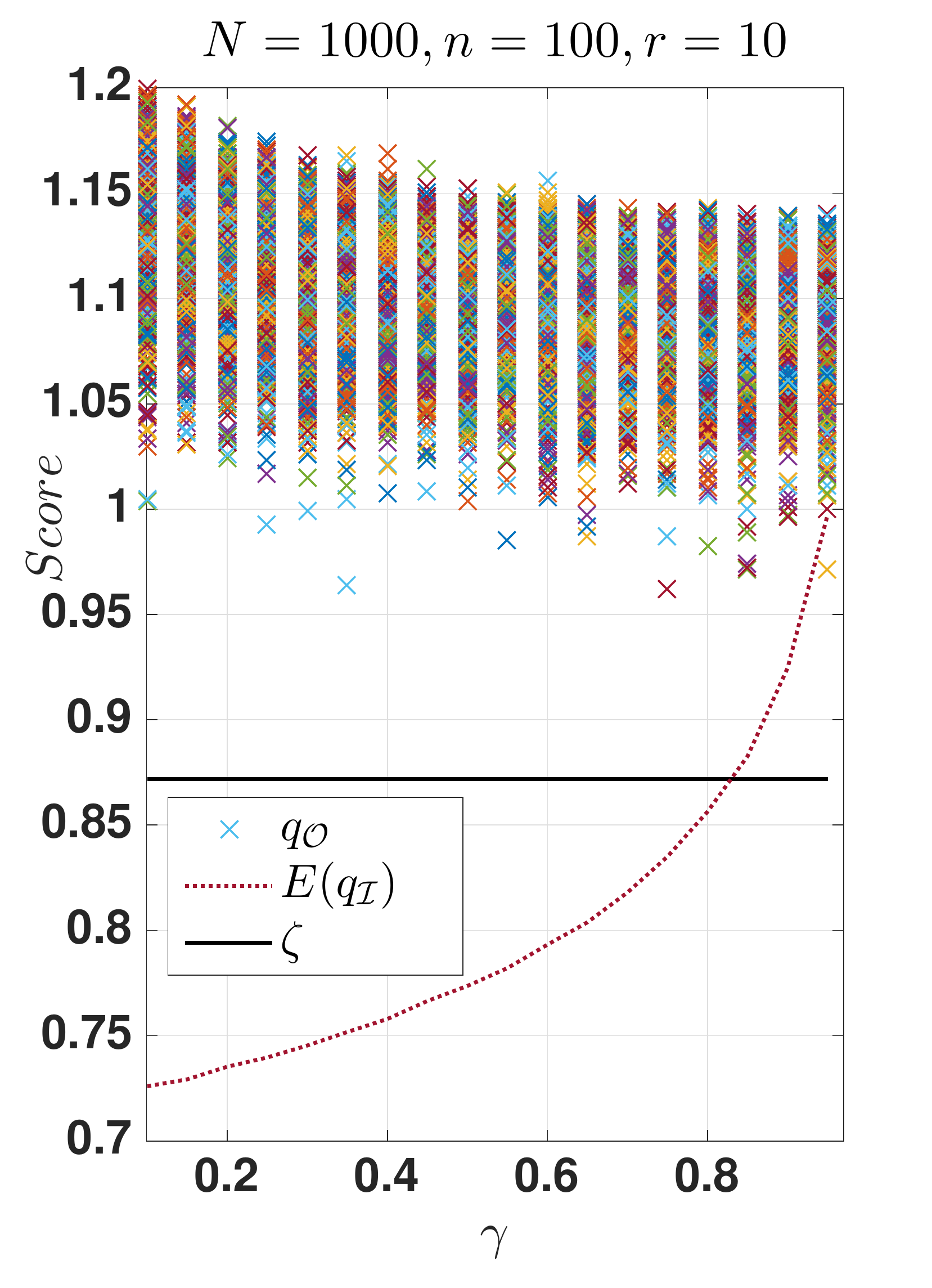}
		\caption{At higher $N$}
		\label{fvalid1}
	\end{subfigure}
	\hfill
	\begin{subfigure}[b]{0.24\textwidth}
		\includegraphics[width=\linewidth, height=6cm]{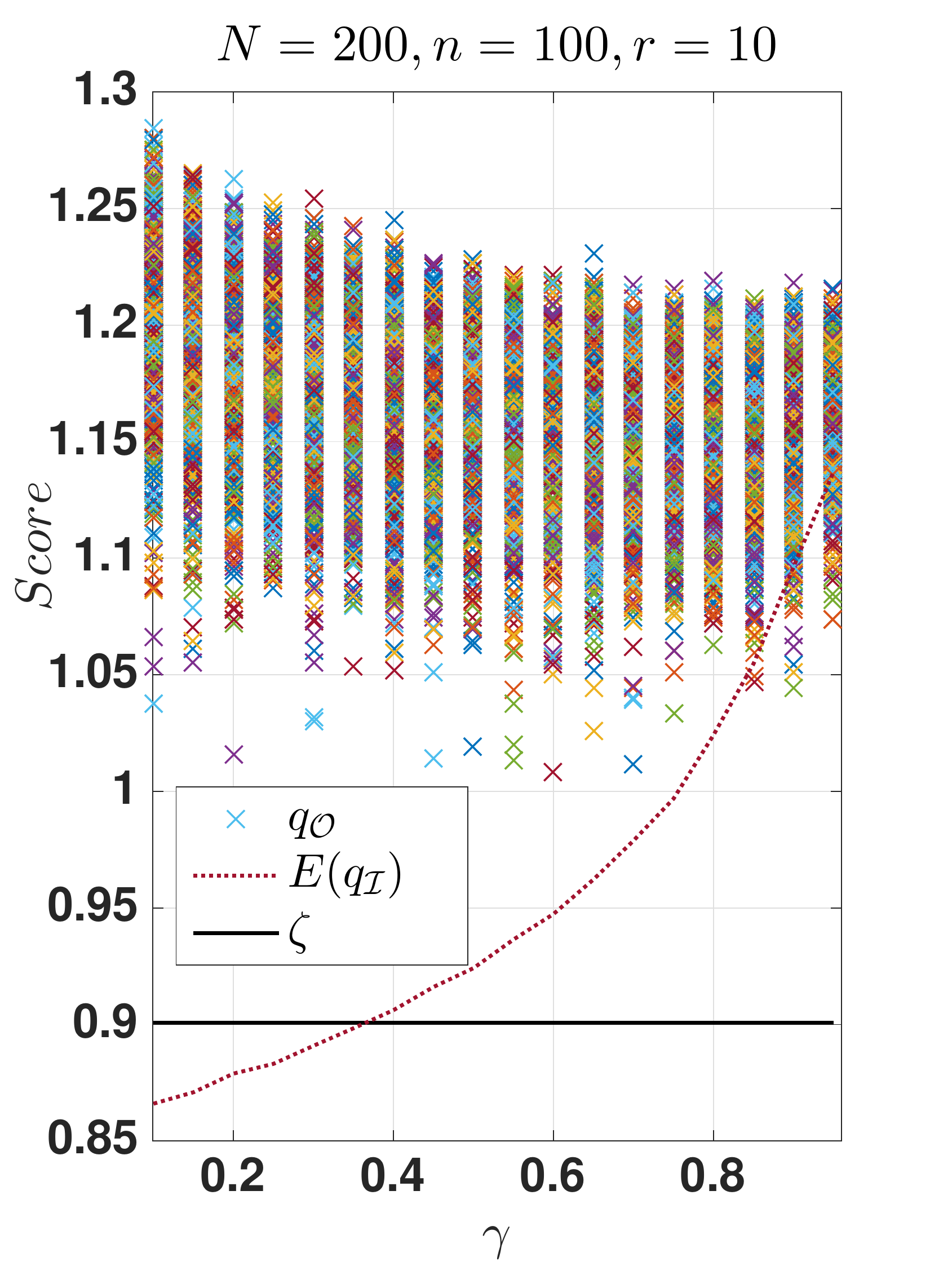}
		\caption{At a lower $N$}
		\label{fvalid2}
	\end{subfigure}
	\hfill
	\begin{subfigure}[b]{0.24\textwidth}
		\includegraphics[width=\linewidth, height=6cm]{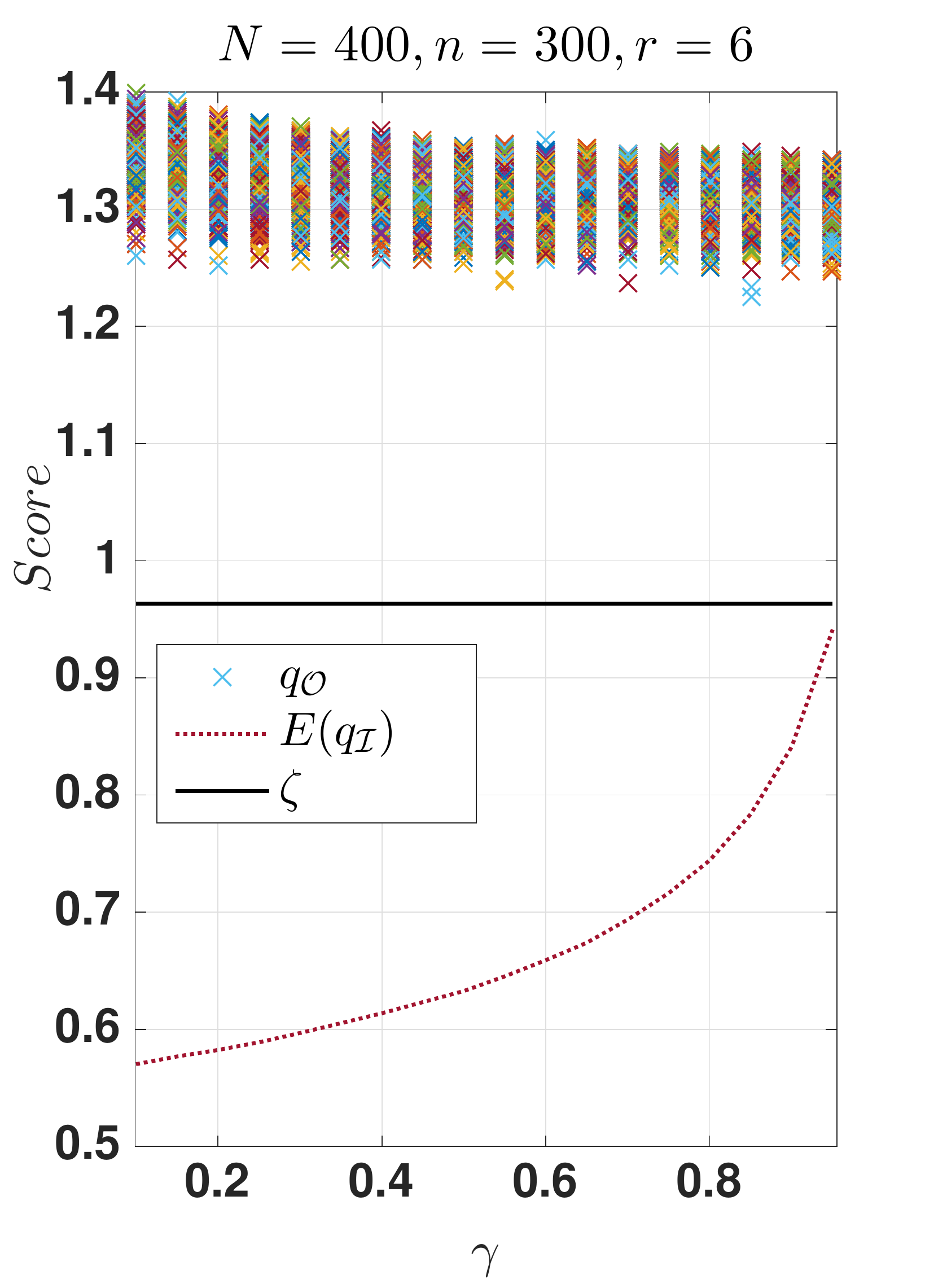}
		\caption{At low $\frac{r}{n}$}
		\label{fvalid3}
	\end{subfigure}
	\hfill
	\begin{subfigure}[b]{0.24\textwidth}
		\includegraphics[width=\linewidth, height=6cm]{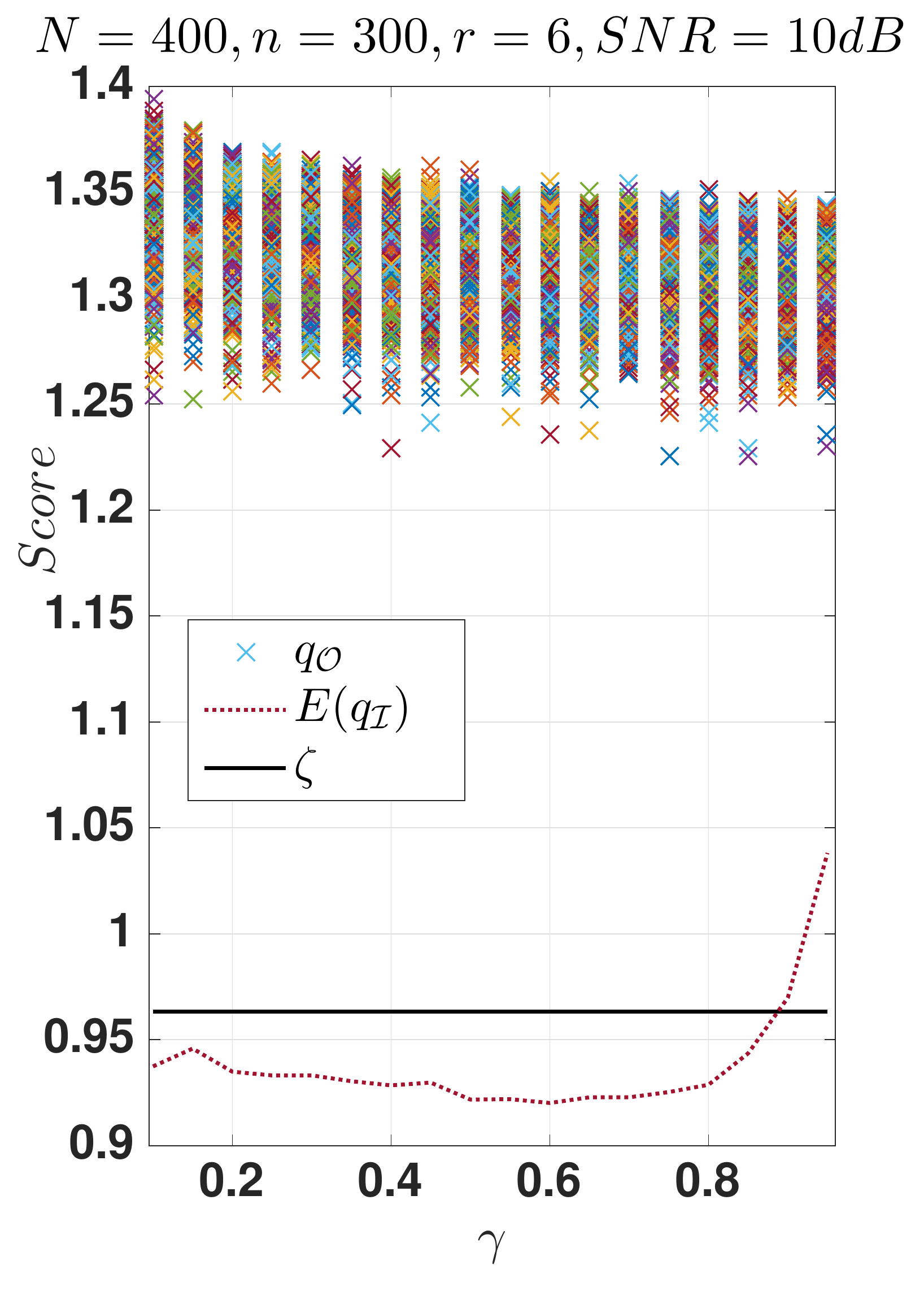}
		\caption{In Gaussian noise}
		\label{fvalidn}
	\end{subfigure}
	\hfill
	\caption{Validating $\zeta$ for different cases of $N, n, r$ over a range of outlier fraction $\gamma$}
	\label{fvalid}
\end{figure*} 
\section{Numerical Simulations}\label{snumsim}
In this section, we present the simulation results of the proposed method. Here we demonstrate the properties of the proposed algorithm in terms of inlier identification and subspace recovery on synthetic and real data. We also compare our method with some of the existing algorithms for robust PCA, in terms of running time of the algorithm and log recovery error ($LRE$) of the estimated subspace, which is defined as in \cite{rahmani2016coherence}, i.e, 
\begin{equation}\label{elre}
LRE =  log_{10}(\dfrac{\|\textbf{U} - \hat{\textbf{U}}\hat{\textbf{U}}^T\textbf{U}\|_F}{\|\textbf{U}\|_F}), 
\end{equation}
where $\textbf{U}$ is basis of the true inlier subspace and $\hat{\textbf{U}}$ is the estimated basis from the algorithms. All our experiments on synthetic data assumes the data model described in this paper under Assumption \ref{amain}. First in this section we provide simulation results to validate the bound proposed in the algorithm.
 \subsection{Validation of bounds}\label{svalid}
Fig \ref{fvalid} plots the value of $q_\mathcal{O}$ against a range of outlier fraction $\gamma$ from $0.1$ to $0.95$. The value of $q_\mathcal{O}$ was calculated for $1000$ trials and plotted. As seen in the figure, the lower bound $\zeta$ holds for all the trials, which means the algorithm always has OIP($\alpha$) regardless of other parameters. The bound $\zeta$ as expected increase with increase in $n$, because as dimension increases, the outlier acute angles will concentrate more towards $\frac{\pi}{2}$. It decreases with $N$, which is also logical because when the number of points increases, the minimum acute angle that an outlier point make will come down. A special case where Gaussian noise $\mathcal{N}(0, \sigma_{w}^2\textbf{I})$ is also added to the system is shown in Fig \ref{fvalidn}. Each data vector $\textbf{m}_i$ was corrupted with noise before normalization and the noise level was kept at $10 dB$ for the experiment. Here too the bound holds in all the trials. Also plotted in the graphs is the expected value of the maximum of inlier scores $q_\mathcal{I}$, calculated over these $1000$ trials. This gives us an indication of ERP($\alpha$). When the value $\mathbb{E}(q_\mathcal{I})$ crosses over the threshold $\zeta$, it means that the inliers recovered will be smaller in number. As seen, when the ratio $\frac{r}{n}$ is small, the algorithm has ERP($\alpha$) even at very low inlier fraction (Fig \ref{fvalid3}). Also when the number of samples increases, for the same value of $r$ and $n$, the inlier recovery performance improves across a higher range of $\gamma$ as seen when comparing Fig \ref{fvalid2} and \ref{fvalid1}. The effect of noise in the system can be observed by comparing Fig \ref{fvalid3} and \ref{fvalidn}. $\mathbb{E}(q_\mathcal{I})$ plot is higher in Fig \ref{fvalidn} which indicates that lesser number of inliers are recovered in the noisy case when compared to the noiseless case.
\subsection{Phase transitions}
\begin{figure}[h]
	\includegraphics[width=8 cm, height = 6 cm]{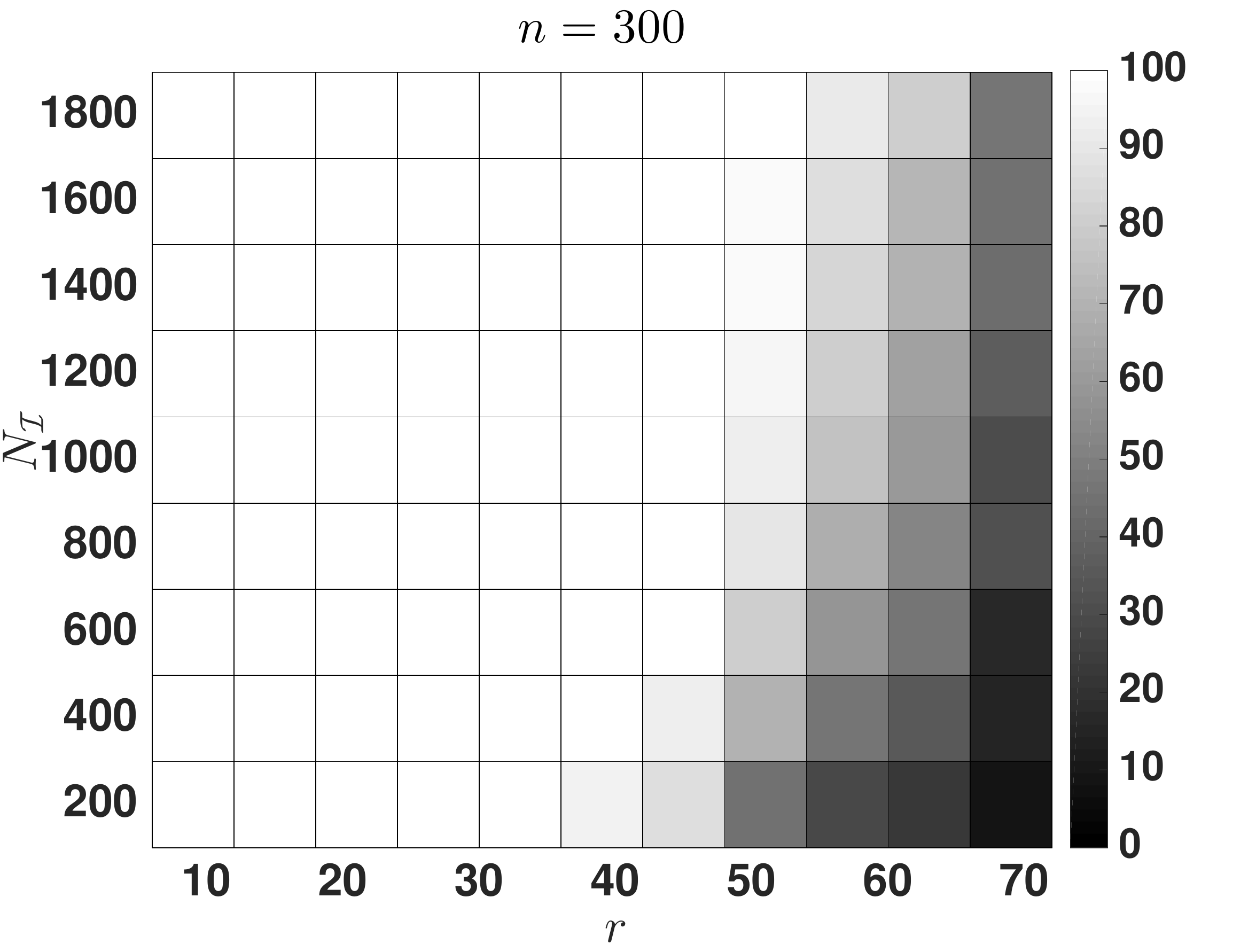}
	\caption{Phase transition plot of percentage of inliers recovered against $\frac{r}{n}$ and $N_\mathcal{I}$}
	\label{finlierrec}
\end{figure}
In this section, we look at the properties of the proposed algorithm in terms of percentage of inliers recovered and error in subspace recovery. First we look at the percentage of inliers recovered by ROMA against varying $\frac{r}{n}$ and the number of inliers $N_{\mathcal{I}}$, as these two are the critical parameters that determines the inlier recovery property of the algorithm. Fig \ref{finlierrec} shows the phase transition on inlier recovery. White indicates $100\%$ inlier recovery and as the squares become darker, the inlier recovery becomes more poor. For this experiment we have set $n=300$ and varied $\frac{r}{n}$ from $0.02$ to $0.24$, i.e $r$ varies from $6$ to $72$, along with varying the number of inliers from $N_\mathcal{I} = 100$ to $1900$ keeping the total number of points $N=2000$. As can be interpreted from the figure, a very high percentage of inliers are recovered for even very small $N_\mathcal{I}$, when $\frac{r}{n}$ is sufficiently low. As $\frac{r}{n}$ increases, the percentage of inliers recovered decreases. This simulation result agrees with Lemma \ref{lrough}. Evaluating the requirement in Lemma \ref{lrough}, for $n=300$, whenever $r< 24$, a large portion of inliers are recovered even at very low inlier fractions. In Fig \ref{finlierrec}, even for $N_\mathcal{I}=100$, the inlier recovery is $100\%$ when the condition in Lemma \ref{lrough} is satisfied by $r$.
\begin{figure}[h]
	\includegraphics[width=8 cm, height = 6 cm]{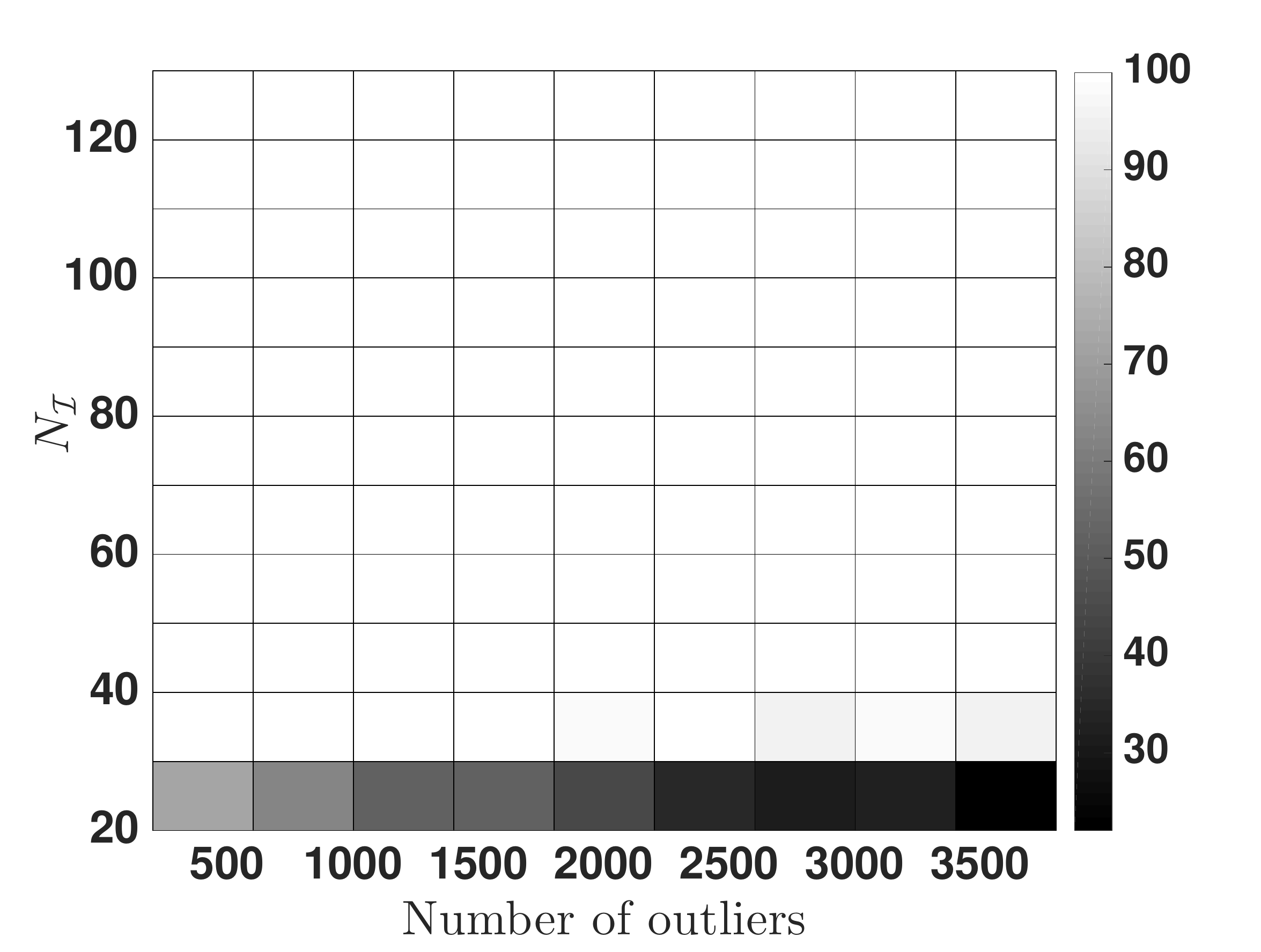}
	\caption{Phase transition plot of exact subspace recovery against $N_\mathcal{I}$ and number of outliers}
	\label{flre}
\end{figure}
In the next experiment we have looked at the subspace recovery property of the algorithm. Here we have considered the noiseless case. After removing outliers through ROMA, the subspace is recovered after doing SVD on the remaining points and choosing the left singular vectors corresponding to the non zero singular values as the recovered subspace basis. A subspace is said to be recovered when $LRE < -5$ for the estimated subspace. We have done $100$ trials each on a range of values of the number of inliers $N_\mathcal{I}$ and number of outliers. Fig \ref{flre} plots the percentage of trials in which the true subspace was recovered against $N_\mathcal{I}$ and the number of outliers, with white indicating $100\%$ success. For this phase transition plot, we have set $n=100$ and $r=10$. It is evident from Fig  \ref{flre} that, whenever there is a good enough number of inliers, no matter what the number of outliers is and for a smaller value of $\frac{r}{n}$, which is $0.1$ in the plot, the subspace is recovered with minimal error. The subspace recovery suffers when the number of inliers is low, as seen in the last row of Fig \ref{flre}. This can be compared with a similar plot in \cite{rahmani2016coherence}, even in the worst case when $\frac{N_\mathcal{I}}{r} < 4$, where \cite{rahmani2016coherence} fails, the proposed algorithm has better inlier recovery over a wider range of the number of outliers.  
\begin{table*}[h]
	\caption{Comparison of Algorithms}
	\label{table1}
	\begin{tabularx}{\textwidth}{@{}l*{10}{C}c@{}}
		\toprule
		Algorithm     &$LRE$ at $\gamma=0.25$&$LRE$  at $\gamma=0.6$ & $LRE$ at $\gamma=0.95$ & Average running time in seconds & Parameter knowledge & Free parameters  \\ 
		\midrule
		FMS  & -14.907      & -14.943          &-13.825   & 0.6414  & $r$   & $>1$  \\ 
		GMS & -$10^{-4}$   & -$10^{-4}$        &-$10^{-5}$    & 0.0759  & $r$    & Regularization   \\ 
		ORSC  & -14.936     & -14.929      &-14.918   & 100.88  & None   & $>1$   \\ 
		CoP  & -14.958      & -14.964          &-14.945     & 0.022  & $\gamma$ or $r$    & No   \\ 
		ROMA & -14.922     & -14.924          &-14.947   & 0.0546  & None   & $\alpha$   \\ 
		\bottomrule
	\end{tabularx}
\end{table*}
\subsection{Advantage of being parameter free}
Here, we highlight through a small experiment why assuming the knowledge of parameters becomes tricky in certain scenarios and how ROMA manages to avoid this pitfall. In Coherence Pursuit (CoP) \cite{rahmani2016coherence}, after arranging the points in terms of decreasing scores, the algorithm picks the first $n_s$ points to recover the subspace through PCA. Choosing $n_s$ requires the knowledge on an upper bound on the number of outliers as it is required that $n_s \leq N_\mathcal{I}$ and also should be reasonably larger than $r$ for successful subspace recovery.
\begin{figure}[h]
	\begin{subfigure}[b]{0.24\textwidth}
		\includegraphics[width=\linewidth, height=4cm]{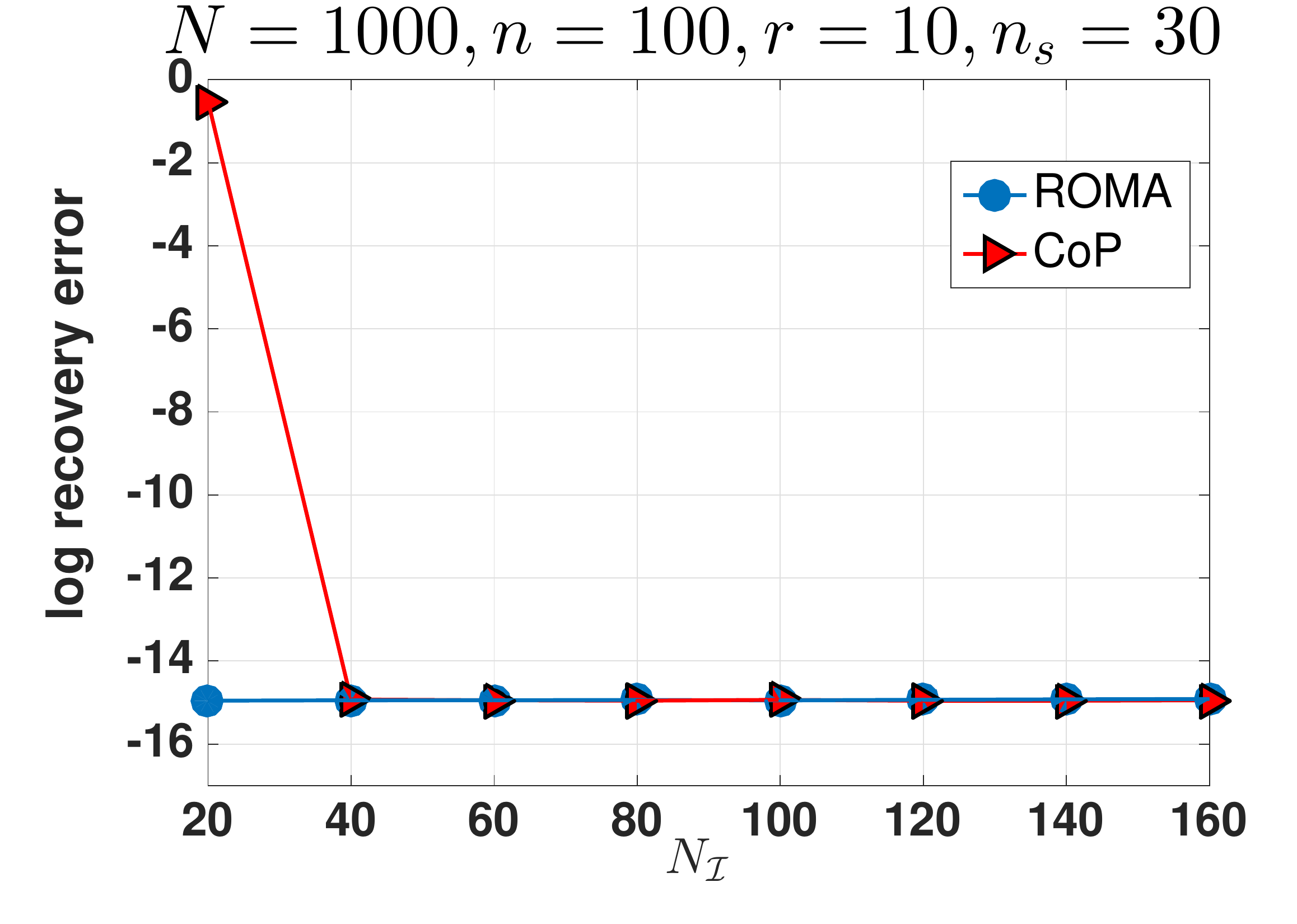}
		\caption{$n_s$ for CoP set to 30}
		\label{ftf1}
	\end{subfigure}
	\hfill
	\begin{subfigure}[b]{0.24\textwidth}
		\includegraphics[width=\linewidth, height=4cm]{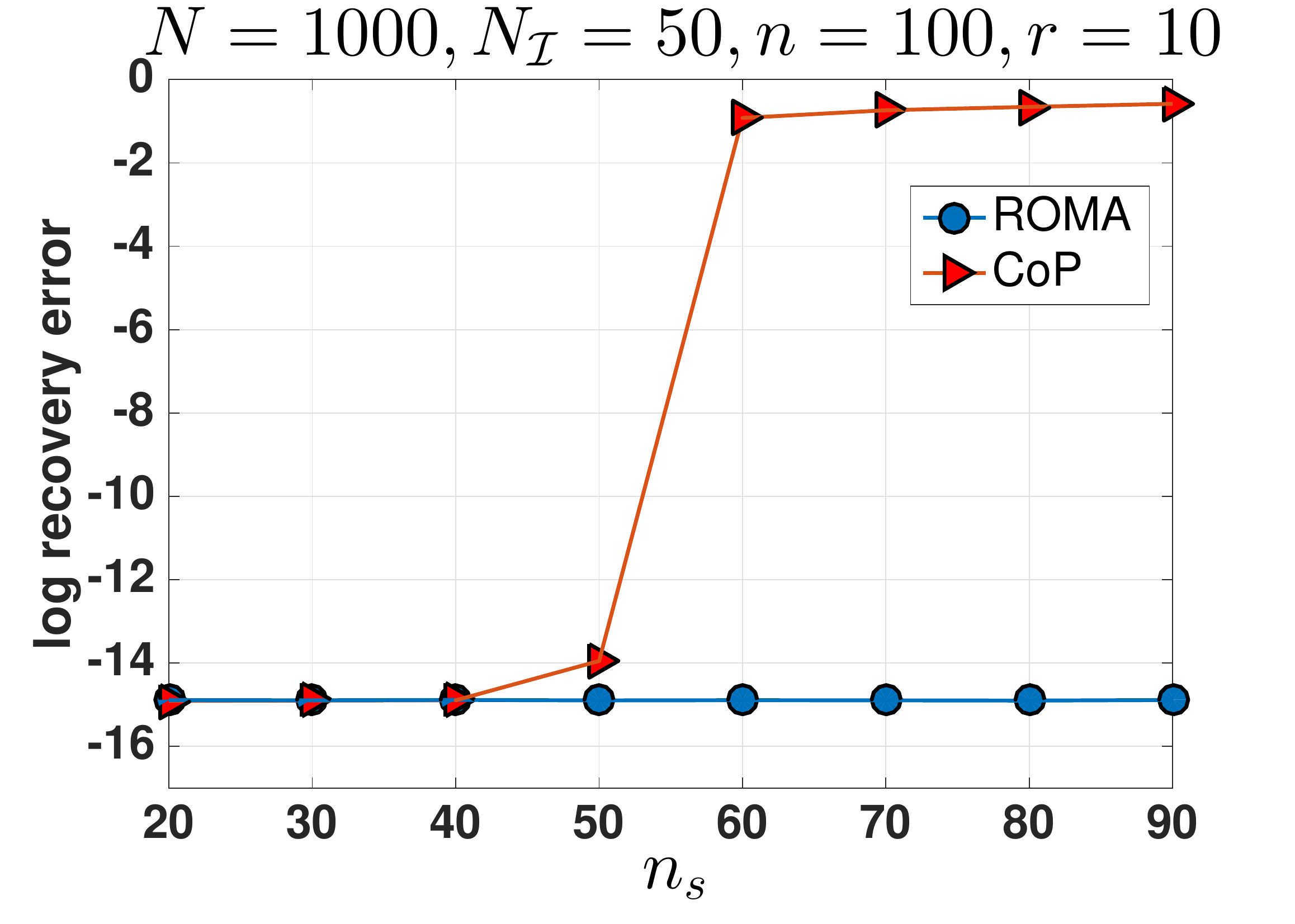}
		\caption{$N_\mathcal{I}$ set to 50, varying $n_s$}
		\label{ftf2}
	\end{subfigure}
	\caption{Demonstration of advantage of parameter free nature of ROMA}
	\label{ftf}
\end{figure}
Here we run two experiments - first for different values of $N_\mathcal{I}$ with the number of points chosen to form the subspace in CoP \cite{rahmani2016coherence} as $n_s = 30$. The results are shown in Fig \ref{ftf1}. As seen in the figure, the performance of CoP deteriorates heavily in terms of $LRE$, when the number of inliers $N_\mathcal{I}$ goes below the value of $n_s$, i.e when the assumption about outlier fraction is wrong. This is highlighted through the second experiment, where we show how $n_s$ affects the performance when $N_\mathcal{I}$ is fixed. When $n_s$ goes above $N_\mathcal{I}$, outlier points are also chosen to form the subspace which corrupts the estimated subspace badly, as seen in Fig \ref{ftf2}. In all the cases the performance of ROMA is unaffected since it is parameter free. We have set $n=100$ and $r=10$ for these experiments. This exercise highlights the importance of an algorithm being parameter free and the effects of an incorrect estimate of a parameter on the performance of a parameter dependent algorithm such as CoP.
\subsection{Comparison with other state of the art algorithms}
Here we compare the proposed algorithm with existing techniques CoP \cite{rahmani2016coherence}, Fast Median Subspace, FMS \cite{lerman2014fast}, Geometric Median Subspace GMS \cite{zhang2014novel} and the outlier removal algorithm outlier removal for subspace clustering (denoted by ORSC for convenience) in \cite{soltanolkotabi2012geometric}, in terms of the log recovery error and running time. For FMS, we used algorithm 1 in \cite{lerman2014fast}, with default parameter setting i.e $p=1,\epsilon =10^{-10}$, maximum iterations 100. For GMS as well, we used the default parameter settings and chose the last $r$ columns from the output matrix as the basis of the estimated subspace. For CoP, we implemented the first method proposed, where we used the number of data points chosen for subspace recovery as $n_s = 30$, which is a value always less than the number of inliers in our experimental settings and hence works well. For ORSC, we used the algorithm using primal-dual interior point method from the $l_1$ magic code repository \cite{candes2005l1} for solving the underlying $l_1$ optimization problem. The parameters used were changed from the default settings to improve convergence rate without degrading the performance. In Table \ref{table1}, we have summarized each algorithm in terms of its performance measured in terms of log recovery error at various outlier fractions, running time and the parameters used by the algorithm for its working. Also the last column indicates other free parameters that an algorithm requires like regularization or convergence parameters. For the experiments in Table \ref{table1}, we have set the values as $N=1000,n=100,r=10$. It is observed that ROMA performs at par with the existing methods in terms of $LRE$ without requiring the knowledge of $r$ or $\gamma$ and is nearly as quick as CoP. The algorithm ORSC which also does not use parameter knowledge, has similar $LRE$ values but is much slower compared to ROMA and also requires multiple parameters like a convergence criterion for solving the underlying $l_1$ optimization problem.
\begin{figure}[h]
	\includegraphics[width=8 cm, height = 6 cm]{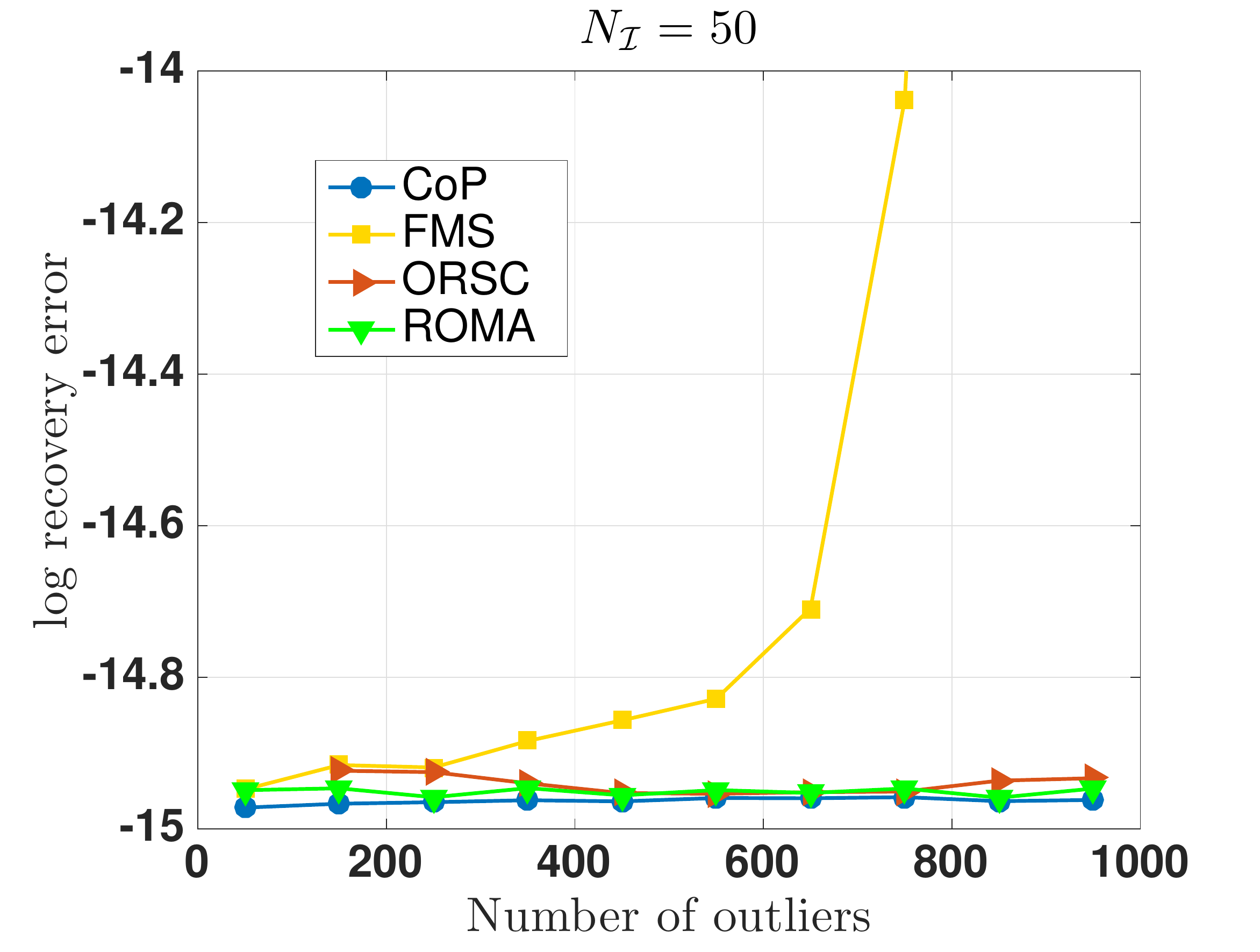}
	\caption{Log recovery error comparison between ROMA and existing algorithms}
	\label{f1}
\end{figure}
Fig \ref{f1} shows a graph on the $LRE$ of algorithms against the number of outliers in the system when the number of inliers are fixed. For this graph we have used $n=100$ and $r=20$, and also fixed the number of inliers to $N_\mathcal{I} =50$. It is seen that ROMA performs at par with all the other state of the art methods in this scenario as well.
\subsection{Performance on Real data}
\begin{figure}[h]
	\begin{subfigure}[b]{0.24\textwidth}
		\includegraphics[width=\linewidth, height=1.25cm]{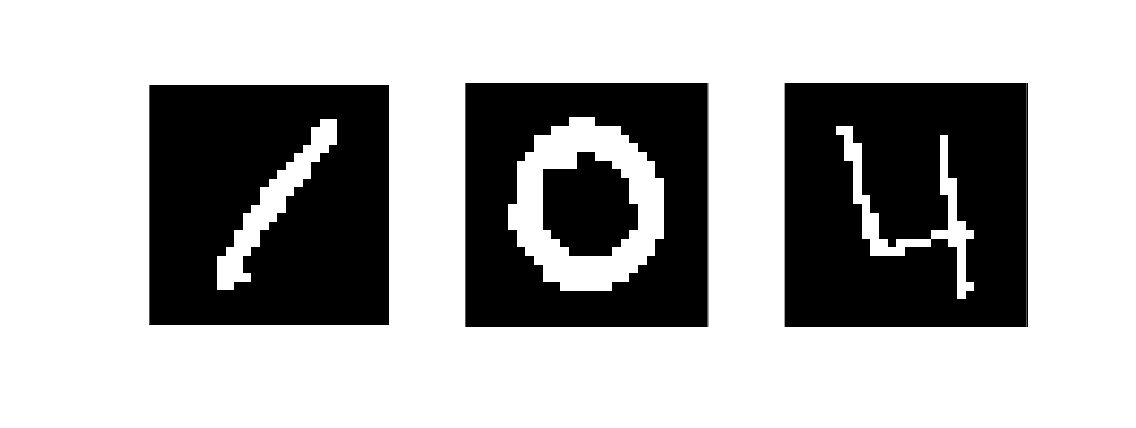}
		\caption{Original digit samples}
		\label{}
	\end{subfigure}
	\hfill
	\begin{subfigure}[b]{0.24\textwidth}
		\includegraphics[width=\linewidth, height=1.25cm]{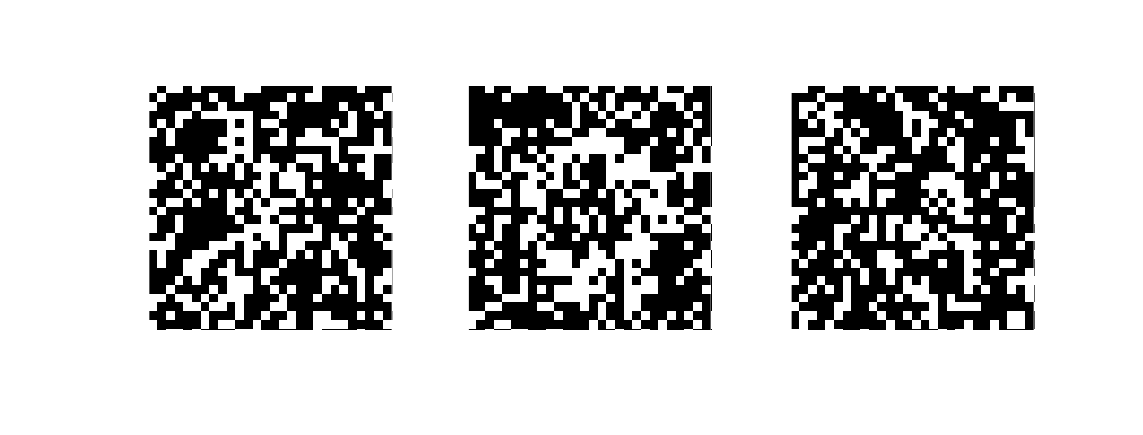}
		\caption{Corrupted digit samples}
		\label{}
	\end{subfigure}
	\caption{Real data experiment: Inliers - original, Outliers - Corrupted samples}
	\label{fapp1}
\end{figure}
\begin{figure}[h]
		\includegraphics[width=8 cm, height = 5 cm]{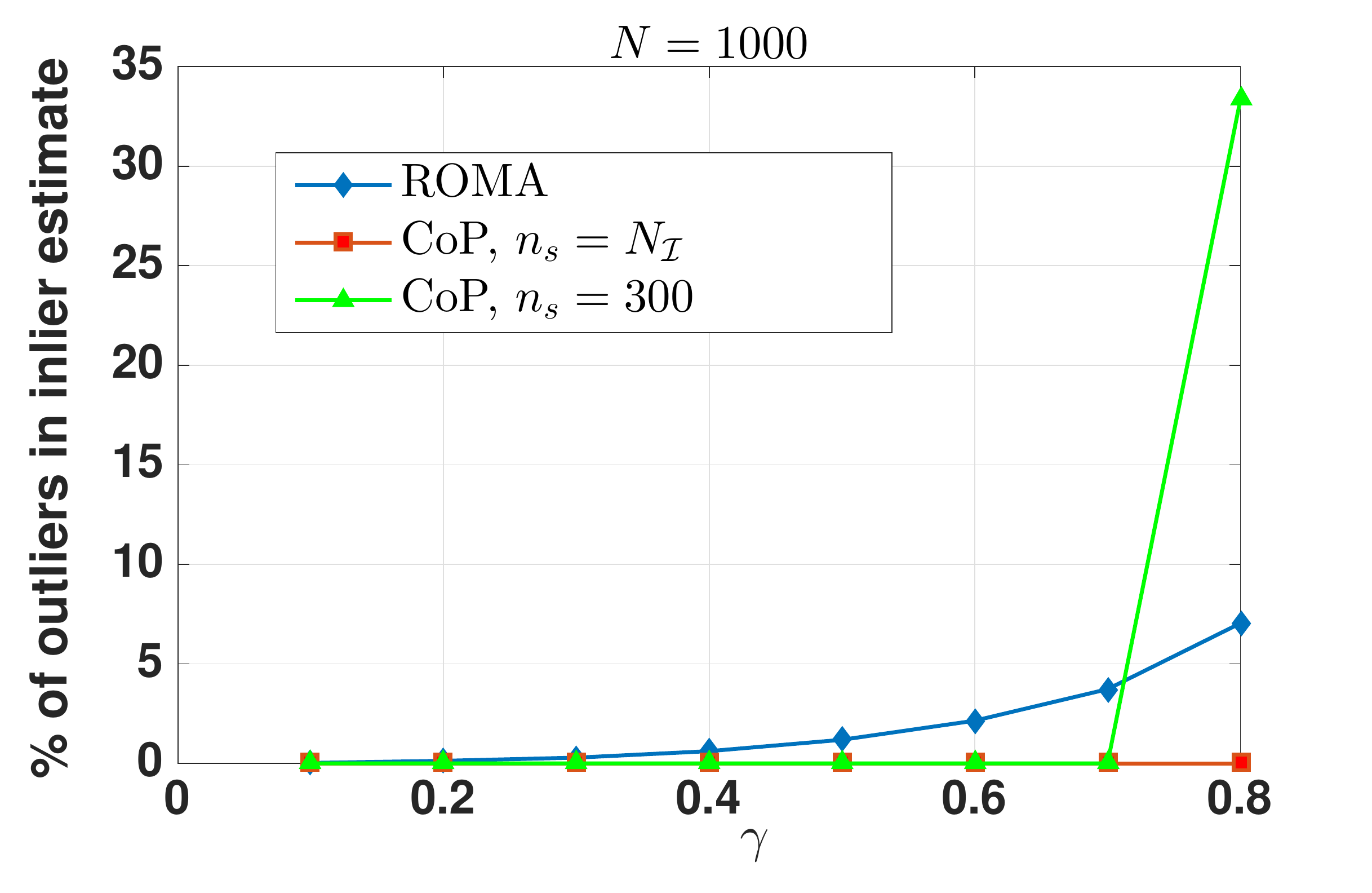}
	\caption{Performance on digit dataset - outliers wrongly classified as inliers}
	\label{fapp2}
\end{figure}
Here we perform a small experiment on real data. In this experiment, we choose a random subset of size $N=1000$ of the well-known MNIST\footnote{http://yann.lecun.com/exdb/mnist/index.html} handwritten digit data set ($28\text{x}28$ images) and corrupt a fraction $\gamma$ of them by adding heavy Gaussian noise. The samples of corrupted and uncorrupted data are given in Fig \ref{fapp1}. The corrupted points are considered outliers and the uncorrupted data points are considered inliers. The aim of the experiment is to recover all inliers and remove outliers. For processing, each image was converted to a vector of dimension $28\times28=784$ and the pixel values were rescaled from the range $[0,255]$ to $[-128,127]$ by subtracting $128$ from each. Hence the data matrix $\textbf{M }\in \mathbb{R}^{784\text{x}1000}$ has $\gamma\times1000$ outlier points. Then we perform outlier removal on this matrix $\textbf{M}$ using both ROMA and CoP and compare the results. For CoP we experimented with three cases of the parameter $n_s$ ($n_s =N_\mathcal{I}$, the number of inliers, $n_s=200$ and $n_s=300$) given as input to the algorithm. We varied the value of $\gamma$ from $0.1$ to $0.8$, i.e the number of inliers vary from $N_\mathcal{I} = 900$ to $N_\mathcal{I} =200$. Note that this experiment setting does not conform to Assumption \ref{amain} since neither the inliers come from a very low rank sub space nor the outliers after normalization follow a uniform distribution in $\mathbb{S}^{783}$, hence the theoretical guarantees need not hold. We performed the experiment over $1000$ trials with different data set in each trial and averaged the results. For all the cases of $\gamma$, all the inliers were recovered by both ROMA and CoP with $n_s = N_\mathcal{I}$. For other values of $n_s$, CoP exactly recovered $n_s$ inliers. We also compared the percentage of outliers in the inlier estimate for both the algorithms Fig \ref{fapp2}. ROMA, in worst case at $\gamma=0.8$, has $\approx 7 \%$ outliers in its inlier estimate and when $\gamma$ is low, this value is very close to $0$. Even though CoP has $0\%$ in all the cases when $n_s \leq N_\mathcal{I}$ ($n_s=N_\mathcal{I}$ case shown in Fig \ref{fapp2}), when $n_s > N_\mathcal{I}$, the outlier content in the inlier estimate goes up to $33\%$. Considering that $N_\mathcal{I}$ is unknown and that in an unsupervised scenario $n_s$ cannot be set by cross validation, this experiment also highlights the importance of the parameter free nature of ROMA. 
\section{Conclusion and future research}\label{s4}
In this paper a simple, fast, parameter free algorithm for robust PCA was proposed, which does outlier removal without assuming the knowledge of dimension of the underlying subspace and the number of outliers in the system. The performance was analyzed both theoretically and numerically. The importance of this work lies in the parameter free nature of the proposed algorithm since estimating unknown parameters or tuning for free parameters in an algorithm is a cumbersome task. Here the popular outlier model was considered with the outliers assumed to be sampled uniformly from a high dimensional hypersphere and the inliers  lying in a single low dimensional subspace. This may be considered as a starting point for developing parameter free algorithms for more complicated data models like ones with structured outliers, inliers sampled from a union of subspaces etc.
\appendices
\section{Some useful Lemmas, proofs of Lemmas \ref{lthetao} and \ref{lthetai}}\label{app1}

\begin{lemma}[Lemma 12 from \cite{cai2013distributions}]\label{l1}
	Let $\textbf{x}_1, \textbf{x}_2... \in \mathbb{S}^{n-1}$ be random points independently chosen with uniform distribution in $\mathbb{S}^{n-1}$, and let $\theta_{ij}$ be defined as in equation \ref{etheta}, then probability density function of $\theta_{ij}$ is given by
	\begin{equation}\label{e4}
	h(\theta) = \dfrac{1}{\sqrt{\pi}}\dfrac{\Gamma(\frac{n}{2})}{\Gamma(\frac{n-1}{2})} (sin\theta)^{n-2} \hskip20pt \theta \in [0, \pi]
	\end{equation}
\end{lemma}
Lemma \ref{l1} implies that the angles $\theta_{ij}$ are identically distributed $\forall i, j$, $i\neq j$ with the pdf $h(\theta)$. The expectation of this distribution is $\frac{\pi}{2}$ and the angles concentrate around $\frac{\pi}{2}$ as $n$ grows. Using the results in \cite{cai2013distributions}, we will state the following:
\begin{remark}\label{r1}
	$h(\theta)$ can be approximated by the pdf of Gaussian distribution with mean $\frac{\pi}{2}$ and variance $\frac{1}{n-2}$, for higher dimensions specifically for $n \geq 5$. In fact $\theta_{ij}$ converges weakly in distribution to $\mathcal{N}(\frac{\pi}{2},\frac{1}{n-2})$ as $n \to \infty$. 
\end{remark}
The remark has been validated in \cite{cai2013distributions}. 
Using the above results now we can prove Lemmas \ref{lthetao} and \ref{lthetai}.
\begin{proof}[\textbf{Proof of Lemma \ref{lthetao}}]
	Now, in our problem setting, by using Assumption \ref{amain}, we have set the outliers to be chosen uniformly at random from all the points in $\mathbb{S}^{n-1}$ and the subspace $\mathcal{U}$ is also chosen uniformly at random. Hence to an outlier, all the other points are just a part of a set of uniformly chosen independent points in $\mathbb{S}^{n-1}$. The results in the lemma then follows directly from Lemma \ref{l1}, Remark \ref{r1} and Assumption \ref{amain}.
\end{proof}
\begin{proof}[\textbf{Proof of Lemma \ref{lthetai}}]
	Assuming that the dimension of the subspace where the inliers lie, $r$ is also large enough i.e $r >5$, the points within a subspace are uniformly chosen points in the hypersphere $\mathbb{S}^{r-1}$ and their distribution is as in equation (\ref{e4}) with $n$ replaced by $r$. Hence from Lemma \ref{l1}, Remark \ref{r1} and Assumption \ref{amain}, the lemma is proved.
\end{proof}
Finally in this section, we also state an important result from \cite{cai2013distributions} used in the proof of Theorem \ref{tthr}. 
\begin{lemma}[Theorem 2 from \cite{cai2013distributions}]\label{lminangle}
	Suppose there are $p$ vectors selected uniformly at random from the unit hypersphere $\mathbb{S}^{n-1}$ Let $\theta_{min}^p$ denote the minimum pairwise angle amongst them i.e $\theta_{min}^p = \underset{1\leq i < j \leq p}{\min}\text{	}\theta_{ij}$	and let $\theta_{max}^p  = \underset{1\leq i < j \leq p}{\max}\text{	}\theta_{ij}$, where $\theta_{ij}$'s are as defined in (\ref{etheta}). Then both $p^{\frac{2}{n-1}}\theta_{min}^p$ and $p^{\frac{2}{n-1}}(\pi-\theta_{max}^p)$ converge weakly as $p \to \infty$, to the distribution given by:
	\begin{align}
	F(x)&=\begin{cases}
	1-e^{-Kx^{n-1}} & \text{for } x\geq 0\\
	0 & \text{for } x < 0\\
	\end{cases}
	\end{align}
	where 
	\begin{equation}\label{eK}
	K = \dfrac{1}{\sqrt{4\pi}}\dfrac{\Gamma(\frac{n}{2})}{\Gamma(\frac{n+1}{2})}
	\end{equation}
\end{lemma}
This lemma gives us the asymptotic distribution of the minimum principal angle amongst all the angles formed between any two points from a set of $p$ uniformly chosen random points on the hypersphere $\mathbb{S}^{n-1}$. Hence as $p \to \infty$, $\mathbb{P}(p^{\frac{2}{n-1}}\theta_{min}^p\leq y) =  \mathbb{P}(p^{\frac{2}{n-1}}(\pi - \theta_{max}^p)\leq y) = 1-e^{-Ky^{n-1}}$. Equivalently, $\mathbb{P}(\theta_{min}^p\leq \frac{y}{p^{\frac{2}{n-1}}}) =  \mathbb{P}((\pi - \theta_{max}^p)\leq \frac{y}{p^{\frac{2}{n-1}}}) = 1-e^{-Ky^{n-1}}$. Use $x =  \frac{y}{p^{\frac{2}{n-1}}}$ and get $\mathbb{P}(\theta_{min}^p\leq x) =  \mathbb{P}(\pi - \theta_{max}^p\leq x) = 1-e^{-K p^2 x^{n-1}}$.
\section{Properties of $\phi_{ij}$}\label{app3}
Here, we will look at $\phi_{ij}$ as defined in (\ref{ephi}). Before we characterize its distribution, we prove the following result:
\begin{lemma}\label{ltphi}
	Let a collection of indices be denoted as $\mathcal{J} = \{(i, j)\text{	}|\text{	}i \in \mathcal{J}_1, j \in \mathcal{J}_2, \text{	} i \neq j\}$, where $\mathcal{J}_1$ and $\mathcal{J}_2$ are any index sets. Let $0 \leq x \leq \frac{\pi}{2}$. 
	Let $E^{\phi}$ denote event that $\underset{\mathcal{J}}{\min} \text{		}\phi_{ij} \leq x$ 
	and let $ E^{\theta}_{min}$ denote event $\underset{\mathcal{J}}{\min}\text{		}\theta_{ij} \leq x$ and $ E^{\theta}_{max}$ denote the event $\pi-\underset{\mathcal{J}}{\max}\text{		}\theta_{ij} \leq x$, then the following are true
	\begin{equation}\label{eangle}
	\begin{aligned}
	\text{a)	}&\mathbb{P}(E^{\phi}) \leq \mathbb{P}(E^{\theta}_{min}) +\mathbb{P}(E^{\theta}_{max})\hskip20pt
	\text{b)	}\mathbb{P}(E^{\phi}) \geq \mathbb{P}(E^{\theta}_{min}) \\
	\text{c)	}&\mathbb{P}((E^{\phi})^c) \geq \mathbb{P}((E^{\theta}_{min})^c) +\mathbb{P}((E^{\theta}_{max})^c)-1\\
		\text{d)	}&\mathbb{P}((E^{\phi})^c) \geq \mathbb{P}((E^{\theta}_{min})^c) \\
	\end{aligned}
	\end{equation}
\end{lemma}
\begin{proof}
	Let  $\underset{\mathcal{J}}{\min} \text{		}\phi_{ij}$ occur at a pair $i^*, j^*$. Then $\phi_{i^*j^*} \leq x$, occurs when either $\theta_{i^*j^*} \leq x$ or $\pi - \theta_{i^*j^*} \leq x$. This is a direct result of the definition of $\phi_{ij}$ as in equation (\ref{ephi}). 
	Also note that in this case, either $\theta_{i^*j^*} $ is the minimum value amongst $\theta_{ij}$'s or $\theta_{i^*j^*} $ is the maximum value amongst $\theta_{ij}$'s. Suppose it is neither, then $\exists$ some $i', j' \in \mathcal{J}$ such that $\theta_{i^*j^*} >  \theta_{i'j'}$ and also some $i'', j''$ such that $\theta_{i^*j^*} <  \theta_{i''j''}$. Then if $\theta_{i^*j^*} \leq \frac{\pi}{2}$, then clearly $\phi_{i^*j^*} > \phi_{i'j'}$, also if $\theta_{i^*j^*} > \frac{\pi}{2}$, then $\phi_{i^*j^*} > \phi_{i''j''}$, which leads to a contradiction.
	Thus we can conjecture that  $E^{\phi} =E^{\theta}_{min} \cup E^{\theta}_{max} $, the results then follow from simple probability union bounds and application of Demorgan's laws.
\end{proof}
\begin{lemma}\label{lphichar}
	The pdf of $\phi_{ij}$ is given by $g(\phi) = 2h(\phi)$, $\phi \in [0, \frac{\pi}{2}]$, where $h(.)$ is the pdf function of the corresponding $\theta_{ij}$. 
\end{lemma}
	\begin{proof}
			$\theta_{ij}$ is the principal angle and hence $0 \leq \theta_{ij} \leq \pi$, then
			\begin{align*}
			\mathbb{P}(\phi_{ij} \leq x) &= \mathbb{P}(\theta_{ij} \leq x \text{	or		} \pi-\theta_{ij} \leq x ) \\
			&\text{Since the events are disjoint for $0 \leq x \leq \frac{\pi}{2}$}\\
			&= \mathbb{P}(\theta_{ij} \leq x) +  \mathbb{P}(\pi-\theta_{ij} \leq x)\\
			&=  \int\limits_{0}^xh(\theta)d\theta + \int\limits_{\pi-x}^{\pi}h(\theta)d\theta  
			\end{align*}
			Since $h(.)$ given by (\ref{e4}) is symmetric around $\frac{\pi}{2}$, 
			\begin{align*}
				\mathbb{P}(\phi_{ij} \leq x) &=	\int\limits_{0}^x2\times h(\phi)d\phi \hskip30pt
				\text{for } 0 \leq x \leq \frac{\pi}{2}.
			\end{align*}  
			Differentiating we get the pdf.
	\end{proof}	
We will use the Gaussian approximation of the pdf $h(.)$ and obtain the approximate mean and variance values for $\phi_{ij}$ using the following lemma. 
\begin{lemma}\label{l3}
	Let $U \sim \mathcal{N}(\mu, \sigma^2)$, define a random variable $V$ as:
	\begin{align*}
	V=\begin{cases}
	U & \text{for } U\leq \mu\\
	2\mu-U & \text{for } U > \mu\\
	\end{cases}
	\end{align*}
	The expectation and variance of $V$ are given by $\mathbb{E}(V) = \mu-\sqrt{\frac{2}{\pi}}\sigma$ and $var(V) = \sigma^2(1-\frac{2}{\pi})$. 
	Also $V > \mu-c\sigma$ w.p $ 2F_{\mathcal{N}}(c)-1$.
\end{lemma}
\begin{proof}The cdf of $V$ is given by, 
	\begin{align*}
	F_V(v) &= \mathbb{P}(V \leq v) = \mathbb{P}(U \leq v \cup 2\mu-U \leq v)\\
	&=\mathbb{P}(U \leq v) + 1- \mathbb{P}(U \leq 2\mu-v)
	\end{align*}
	The pdf is given after differentiation. Note that since the pdf of $U$ is symmetric around $\mu$, $f_U(2\mu-v) = f_U(v)$. Thus $f_V(v) = 2f_U(u)$ for $U\leq \mu$ and $0$ otherwise.The moment generating function of $V$, $M_V(t)$ is thus given by 
	\begin{align*}
	M_V(t) &= \mathbb{E}(e^{Vt})
	= \int\limits_{-\infty}^{\mu}e^{vt}\dfrac{2}{\sqrt{2\pi}\sigma}e^{-\frac{(v-\mu)^2}{2\sigma^2}}dv \\
	&=2e^{\mu t+\frac{\sigma^2t^2}{2}}F_{\mathcal{N}}(-\sigma t) \hskip10pt
	\end{align*}
	where $F_{\mathcal{N}}(.)$ is the standard normal cdf. The last step was using a change of variable in integration $z = \frac{v-\mu}{\sigma} -\sigma t$ and definition of $F_{\mathcal{N}}(.)$. Using MGF, we can easily derive the moments of $V$ as $\mathbb{E}(V) = \Big[\dfrac{d M_V(t) }{dt}\Big]_{t=0} = \mu-\sqrt{\dfrac{2}{\pi}}\sigma$ and $\mathbb{E}(V^2)  = \Big[\dfrac{d^2 M_V(t) }{dt^2}\Big]_{t=0}= \sigma^2+\mu^2-2\mu\sigma\sqrt{\dfrac{2}{\pi}}$, using the result that $\dfrac{dF_{\mathcal{N}}(\sigma t)}{dt} = -\frac{\sigma}{\sqrt{2\pi}}e^{-\frac{\sigma^2 t^2}{2}}$. This also gives the variance as $var(V) = \mathbb{E}(V^2)-(\mathbb{E}(V))^2 = \sigma^2(1-\dfrac{2}{\pi})$.
	\begin{align*}
	    \mathbb{P}(V \leq \mu-k\sigma) &= \int\limits_{-\infty}^{\mu_V-k\sigma_V}2f_U(v)dv =2F_{\mathcal{N}}(-k)
	\end{align*}
	Hence $V > \mu-c\sigma$ w.p $1-2F_{\mathcal{N}}(-c) = 2F_{\mathcal{N}}(c)-1$. 
\end{proof}
\begin{corollary}\label{c2}
	 For $i \in \mathcal{O}$, $\mathbb{E}(\phi_{ij}) \approx \dfrac{\pi}{2} - \sqrt{\frac{2}{\pi(n-2)}}$ and $var(\phi_{ij}) \approx  \dfrac{1-\frac{2}{\pi}}{n-2}$ and $\phi_{ij}>\dfrac{\pi}{2} - \frac{c}{\sqrt{n-2}}$ with probability $2F_{\mathcal{N}}(c) -1$. When $i, j \in \mathcal{I}$, $\mathbb{E}(\phi_{ij}) \approx \dfrac{\pi}{2} - \sqrt{\frac{2}{\pi(r-2)}}$ and $var(\phi_{ij}) \approx  \dfrac{1-\frac{2}{\pi}}{r-2}$ and $\phi_{ij}>\dfrac{\pi}{2} - \frac{c}{\sqrt{r-2}}$ with probability $2F_{\mathcal{N}}(c) -1$.
\end{corollary}
\begin{proof}
	This is a straightforward application of Lemma \ref{l3} to the definition of $\phi_{ij}$ and using the Gaussian approximation for the distributions of $\theta_{ij}$'s
\end{proof}
\section{Proofs of results in section \ref{sguaran}}\label{app2}
\begin{proof}[\textbf{Proof of Lemma \ref{lrough}}]
	We know for an outlier point, the angles made by an outlier has the following property from Corollary \ref{c2}, keeping $c=3$, 
	\begin{equation*}
	\phi_{ij}>\dfrac{\pi}{2} - \frac{3}{\sqrt{n-2}}, \hskip10pt  \forall i \in \mathcal{O}\hskip30pt \text{w.h.p}  
	\end{equation*} 
		We also know that for $i \in \mathcal{I}$, Since $q_i = \underset{j}{\min} \phi_{ij}$, by Jensen's inequality, $\mathbb{E}(q_i) \leq \min\{[\mathbb{E}(\phi_{ij})\text{	}i, j \in \mathcal{I}], [\mathbb{E}(\phi_{ij})\text{	}i\in \mathcal{I}, j \in \mathcal{O}]\}$. Hence using Corollary \ref{c2}, using the fact that $r<n$, 
	\begin{equation*}
	\mathbb{E}(q_i) \leq \dfrac{\pi}{2} - \sqrt{\frac{2}{\pi(r-2)}}
	\end{equation*}
	To have a reasonably good separation between the inlier and outlier scores, the following condition helps
	\begin{equation*}
	\mathbb{E}(q_i) \leq \dfrac{\pi}{2} - \frac{3}{\sqrt{n-2}}
	\end{equation*}
	This is satisfied if
	\begin{equation*}
\dfrac{\pi}{2} - \sqrt{\frac{2}{\pi(r-2)}}< \dfrac{\pi}{2} - \frac{3}{\sqrt{n-2}}
	\end{equation*}
	This when simplified gives the condition in the lemma.
\end{proof}
\begin{proof}[\textbf{Proof of Theorem \ref{trevcond}}]
	We know for the algorithm to follow ERP($\alpha$), we require $\mathbb{P}(q_\mathcal{I} \leq \zeta) \geq 1-\alpha$. Let $i \in \mathcal{I}$. 
	\begin{align*}
		\mathbb{P}(q_\mathcal{I} \leq \zeta) &= \mathbb{P}(\underset{i \in \mathcal{I}}{\max}\text{   }q_i \leq \zeta)= \mathbb{P}(q_1 \leq \zeta, q_2 \leq \zeta, ...q_{N_\mathcal{I}} \leq \zeta) \\
		& \leq \mathbb{P}(q_i \leq \zeta)
	\end{align*}
	Hence if $\mathbb{P}(q_i \leq \zeta) < 1-\alpha$, then the algorithm is sure to have no ERP($\alpha$). 
	\begin{align*}
	\mathbb{P}(q_i \leq \zeta) &= \mathbb{P}({\underset{\underset{j\neq i}{j\in \{1, 2..N\}}}{\min} \phi_{ij} \leq \zeta})= \mathbb{P}({\underset{\underset{j\neq i}{j\in \{1, 2..N\}}}{\cup} \{\phi_{ij} \leq \zeta\}})\\
	&\leq (N_\mathcal{I} -1)\mathbb{P}(\phi_{i, \mathcal{I}} \leq \zeta) + \gamma N \mathbb{P}(\phi_{i, \mathcal{O}} \leq \zeta)
	\end{align*}
	This is by union bound and the index $j$ has been changed to $\mathcal{I}$ and $\mathcal{O}$ to indicate the angles formed by the point indexed by $i$ with an inlier and outlier respectively. Now if this upper bound is less than $1-\alpha$, the algorithm cannot follow ERP($\alpha$). Thus for the algorithm to not follow ERP($\alpha$), we require, 
	\begin{equation*}
		(N_\mathcal{I} -1)\mathbb{P}(\phi_{i, \mathcal{I}} \leq \zeta) + \gamma N \mathbb{P}(\phi_{i, \mathcal{O}} \leq \zeta) < 1- \alpha
	\end{equation*}
	\begin{equation*}
\text{Rearranging, 	}N_\mathcal{I} < 1+\dfrac{1-\alpha}{F_{\phi^{\mathcal{I}}}(\zeta)} - \dfrac{\gamma N F_{\phi^{\mathcal{O}}}(\zeta)}{F_{\phi^{\mathcal{I}}}(\zeta)}, 
\end{equation*}
	where $F_{\phi^{\mathcal{I}}}(.)$ is the cdf function of $\phi_{ij}$, when $i, j \in \mathcal{I}$ and $F_{\phi^{\mathcal{O}}}(.)$ is the cdf function of $\phi_{ij}$, when $i \in \mathcal{I}, j \in \mathcal{O}$. We know the cdf characterization of $\phi_{ij}$ from Lemma \ref{lphichar} and it depends on the distribution of $\theta_{ij}$'s. Also note that $F_{\phi_{ij}}(\phi) = 2F_{\theta_{ij}}(\phi)$ for $0 \leq \phi \leq \dfrac{\pi}{2}$. As we have seen earlier, this cdf can be well approximated by a Gaussian cdf with mean $\frac{\pi}{2}$ and variance $\dfrac{1}{r-2}$, when $j \in \mathcal{I}$ and $\dfrac{1}{n-2}$, when $j \in \mathcal{O}$. 
\end{proof}
\begin{proof}[\textbf{Proof of Lemma \ref{lerpb}}]
	We will use $k$ to index the inlier points in this proof. We are looking at the below probability:
	\begin{align*}
	\mathbb{P}(q_\mathcal{I} \leq \zeta) &= \mathbb{P}(\underset{k \in \mathcal{I}}{\max}\text{   }q_k \leq \zeta)
	= 1 - \mathbb{P}(\underset{k \in \mathcal{I}}{\max}\text{   }q_k > \zeta)\\
	& = 1-  \mathbb{P}(\underset{k \in \mathcal{I}}{\cup}\text{   }\{q_k > \zeta\})\\
	& \geq 1- \sum\limits_{k \in \mathcal{I}} \mathbb{P}(q_k > \zeta)= 1- N_{\mathcal{I}}\mathbb{P}(q_k > \zeta)
	\end{align*}
	Here we apply the union bound, and the final step is because the events $q_k > \zeta$ are equiprobable. Since all the angles are identically distributed, so will their minimum over an index set. Using the definition of $q_k$
	\begin{align*}
	\mathbb{P}(q_\mathcal{I} \leq \zeta) &\geq 1- N_{\mathcal{I}}\mathbb{P}({\underset{j\in \{1, 2..N\}, j\neq k}{\min} \phi_{kj} > \zeta})
	\end{align*}
	Using Lemma \ref{lminangle}, equation \ref{eangle} (d), $\mathbb{P}({\underset{j\in \{1, 2..N\}, j\neq k}{\min} \phi_{kj} > \zeta}) \leq \mathbb{P}({\underset{j\in \{1, 2..N\}, j\neq k}{\min} \theta_{kj} > \zeta})$. Hence
	\begin{align*}
	\mathbb{P}(q_\mathcal{I} \leq \zeta) &\geq 1- N_{\mathcal{I}}\mathbb{P}({\underset{j\in \{1, 2..N\}, j\neq i}{\min} \theta_{kj} > \zeta})
	\end{align*}
\end{proof}
{
	\bibliographystyle{IEEEtran}
	\bibliography{bibl}
}
\end{document}